\documentclass[nohyperref]{article}

\usepackage{amsmath,amsfonts,amsthm,amssymb,ifthen}
\usepackage[utf8]{inputenc} 
\usepackage[T1]{fontenc}    
\usepackage{url}            
\usepackage{booktabs}       
\usepackage{amsfonts}       
\usepackage{nicefrac}       
\usepackage{microtype}      
\usepackage{amsmath}        
\usepackage{amsthm}         
\usepackage{mathtools}      
\usepackage{float}
\usepackage{enumitem}
\usepackage{xspace}
\usepackage{fullpage}
\usepackage[dvipsnames]{xcolor}
\usepackage{hyperref}
\usepackage{pdfsync}

\usepackage{libertine}
\usepackage{libertinust1math}
\usepackage{inconsolata}
\usepackage{algorithmic,algorithm}

\usepackage[noend,ruled,linesnumbered,algo2e]{algorithm2e}

\newcommand{\declarecolor}[2]{\definecolor{#1}{RGB}{#2}\expandafter\newcommand\csname #1\endcsname[1]{\textcolor{#1}{##1}}}
\declarecolor{White}{255, 255, 255}
\declarecolor{Black}{0, 0, 0}
\declarecolor{LightGray}{216, 216, 216}
\declarecolor{Gray}{127, 127, 127}
\declarecolor{Orange}{237, 125, 49}
\declarecolor{LightOrange}{251,229, 214}
\declarecolor{Yellow}{255, 192, 0}
\declarecolor{LightYellow}{255, 242, 200}
\declarecolor{Blue}{91, 155, 213}
\declarecolor{LightBlue}{222, 235, 247}
\declarecolor{Green}{112, 173, 71}
\declarecolor{LightGreen}{226, 240, 217}
\declarecolor{Navy}{68, 114, 196}
\declarecolor{LightNavy}{218, 227, 243}
\declarecolor{DarkBlue}{0,20,115}

\hypersetup{
	colorlinks=true,
	pdfpagemode=UseNone,
	citecolor=Navy,
	linkcolor=Navy,
	urlcolor=Navy,
}

\def\prob#1#2{\mbox{Pr}_{#1}\left[ #2 \right]}

\def\defeq{\stackrel{\mathrm{def}}{=}}

\newcommand\s{f}

\newtheorem{lemma}{Lemma}[section]
\newtheorem{corollary}{Corollary}[section]

\newtheorem{definition}{Definition}[section]

\def\calN{\mathcal{N}}

\def\calX{\mathcal{X}}
\def\calY{\mathcal{Y}}
\def\calD{\mathcal{D}}

\def\calS{\mathcal{S}}

\newcommand\R{\boldsymbol{\mathbb{R}}}

\newcommand\xx{\boldsymbol{\mathit{x}}}

\newcommand{\one}{\mathbf{1}}

\usepackage{parskip}

\usepackage[round]{natbib}
\begin{document}
\title{Heterogeneous Calibration: A post-hoc model-agnostic framework for improved generalization}

\author{
  David Durfee%
  \thanks{LinkedIn Corporation.
  Email: \href{mailto:ddurfee@linkedin.com}{ddurfee@linkedin.com}}
  \and
  Aman Gupta%
  \thanks{LinkedIn Corporation.
  Email: \href{mailto:amagupta@linkedin.com}{amagupta@linkedin.com}}
  \and
  Kinjal Basu
  \thanks{LinkedIn Corporation.
  Email: \href{mailto:kbasu@linkedin.com}{kbasu@linkedin.com}}
}
\date{\today}

\maketitle
\thispagestyle{empty}

\begin{abstract}

We introduce the notion of \textit{heterogeneous calibration} that applies a post-hoc model-agnostic transformation to model outputs for improving AUC performance on binary classification tasks.
We consider \textit{overconfident models}, whose performance is significantly better on training vs test data and give intuition onto why they might under-utilize moderately effective simple patterns in the data. 
We refer to these simple patterns as heterogeneous partitions of the feature space and show theoretically that perfectly calibrating each partition separately optimizes AUC. 
This gives a general paradigm of \textit{heterogeneous calibration} as a post-hoc procedure by which heterogeneous partitions of the feature space are identified through tree-based algorithms and post-hoc calibration techniques are applied to each partition to improve AUC. 
While the theoretical optimality of this framework holds for any model, we focus on deep neural networks (DNNs) and test the simplest instantiation of this paradigm on a variety of open-source datasets. 
Experiments demonstrate the effectiveness of this framework and the future potential for applying higher-performing partitioning schemes along with more effective calibration techniques.

\end{abstract}

\section{Introduction}\label{sec:intro} 
 
Deep neural networks (DNNs) have become ubiquitous in decision making pipelines across industries due to an extensive line of work for improving accuracy, and are particularly applicable to settings where massive datasets are common \cite{ He2016, devlin-etal-2019-bert, attention17, DBLP:journals/corr/abs-1906-00091}.
The large number of parameters in DNNs affords greater flexibility in the modeling process for improving generalization performance and has been recently shown to be necessary for smoothly interpolating the data~\cite{bubeck2021universal}.

However this over-parameterization (where the number of parameters exceed the number of training samples), 
along with other factors can lead to \textit{over-confidence}, where model performance is substantially better on training data compared to test data.
For classification tasks, over-confidence is more specifically characterized by the model output probability of the predicted class being generally higher than the true probability.
\citet{Guo2017} found that over-confidence increased with respect to the model depth and width, even when accuracy improves.
Additional recent work proves that over-confidence is also inherent for under-parameterized logistic regression~\cite{Bai2021don}.
Accordingly, there is extensive work in the area of calibration, whose primary goal is to 
improve the accuracy of probability estimates that is essential for different use-cases.

Some of the common calibration techniques  that apply a post-hoc model-agnostic transformation to properly adjust the model output include Platt scaling~\cite{platt1999probabilistic}, Isotonic regression~\cite{zadrozny2002transforming}, histogram binning~\cite{zadrozny2001obtaining}, Bayesian binning into quantiles~\cite{naeini2015obtaining}, scaling-binning~\cite{kumar2019verified}, and Dirichlet calibration ~\cite{kull2019beyond}.
There is also work on calibration through ensemble type methods~\cite{lakshminarayanan2016simple, Gal16} and recent work on using focal loss to train models that are already well-calibrated while maintaining accuracy \cite{mukhoti2020calibrating}. For a more indepth exposure to different calibration methods, please see \citet{Bai2021don}.




In this paper, we deviate from this traditional aim of calibration. Instead of trying to improve accuracy of probability estimates, we aim at improving model generalization. 
Our key insight is that over-confident models not only show mis-calibration but also tend to under-utilize heterogeneity in the data in a specific and intuitive manner.
We develop this intuition through concrete examples, and focus on mitigating this under-utilization of data heterogeneity through efficient post-hoc calibration.
We specifically aim to improve model accuracy, characterized through the area under the receiver-operating characteristic (ROC) curve, (commonly called AUC) and other metrics for binary classification. 

Accordingly, we develop a new theoretical foundation for post-hoc calibration to improve AUC and prove that the transformation for optimizing output probability estimates will also optimize AUC. To the best of our knowledge, this is the first paper that provably shows how calibration techniques can improve model generalization. We further extend this theoretical optimality to separately calibrating different partitions of the heterogeneous feature space, and give concrete intuition how partitioning through standard tree-based algorithms and separately calibrating the partitions will improve AUC.
This gives a natural and rigorous connection between tree-based algorithms and DNNs through the use of standard calibration techniques, and provides an efficient post-hoc transformation framework to improve accuracy.

In order to best show that the underlying theory of our general framework holds in practice, we test the simplest instantiation whereby a decision tree classifier identifies the partitioning and logistic regression is used for calibration.
We test on open-source datasets and focus upon tabular data due to it's inherent heterogeneity, but also discuss how this can be extended to image classification or natural language processing tasks.
Across the different datasets we see a notable increase in performance from our heterogeneous calibration technique on the top performing DNN models. 
In addition, we see much more substantial increase in performance and more stable results while considering the average performing DNNs in the hyper-parameter tuning.
Our experiments also confirm the intuition that more over-confident models will see greater increase in performance from our heterogeneous calibration framework.

We summarize our contributions as the following:
\begin{enumerate}[noitemsep]
\item We use concrete examples to give intuition on how over-confident models, particularly DNNs, under-utilize heterogeneous partitions of the feature space.
\item We provide theoretical justification of correcting this under-utilization through standard calibration on each partition to maximize AUC.
\item We leverage this intuition and theoretical optimality to introduce the general paradigm of \textit{heterogeneous calibration} that applies a post-hoc model-agnostic transformation to model outputs for improving AUC performance on binary classification tasks. This framework also easily generalizes to multi-class classification.
\item We test the simplest instantiation of heterogeneous calibration on open-source datasets to show the effectiveness of the framework
\end{enumerate}

The rest of the paper is organized as follows. We begin with the detailed problem setup in  Section~\ref{sec:notation}. In Section \ref{sec:intuition} we give an intuition as to how over-confident models tend to under-utilize heterogeneity. 
In Section~\ref{sec:auc_calibration} we give a provably optimal post-hoc transformation for mitigating under-utilized heterogeneity. 
In Section~\ref{sec:framework} we detail the framework of heterogeneous calibration.
In Section~\ref{sec:experiments} we give the experimental results before ending with a discussion in Section~\ref{sec:discussion}. All proofs are pushed to the Appendix. 
\section{Methodology}\label{sec:notation} 
In practice, calibration is primarily used to get better probability estimates, which is especially useful for use-cases where uncertainty estimation is important \cite{Bai2021don}. At first glance, the notion of over-confidence being corrected by post-hoc calibration should only contract the range of probability estimates but not affect the ordering and thereby the AUC. Moreover, much of the literature attempts to maintain accuracy within the calibration. 
However if we consider over-confidence at a more granular level it can negatively affect the relative ordering between \textit{heterogeneous partitions}. 
We primarily consider a heterogeneous partitioning to be a splitting of the feature space such that each partition has a disproportionately higher ratio of positive or negative labels for the binary classification setting. Intuitively, the more accurately a model has intricately fit the data the less it needs to utilize simpler patterns, such as heterogeneous partitions, but over-confident models will over-estimate their ability to fit the data and thus under-utilize simple patterns.


Ideally the relative ordering of heterogeneous partitions could be corrected for over-confident models as a post-hoc procedure.
One approach would be to add a separate bias term to the output of each partition, but this may not fully capture the extent to which the relative ordering can be improved.
We give a more rigorous examination of the AUC metric which measures the quality of our output ordering and prove that perfectly calibrating the probability estimates will also optimize the AUC and several other accuracy metrics for the given model.
Furthermore we show that this extends to any partitioning of the feature space such that perfectly calibrating each partition separately will maximally improve AUC and other related metrics.
The concept of separately calibrating partitions of the feature space has also been seen in the fairness literature~\cite{pmlr-v80-hebert-johnson18a}, but their partitions are predefined based upon fairness considerations and the considered metrics are towards ensuring fair models.

Combining our theoretical result with the intuition that over-confident models will improperly account for heterogeneous partitions gives a general framework of \textit{heterogeneous calibration} as a post-hoc model-agnostic transformation that: (1) indentifies heterogeneous partitions of the feature space through tree-based algorithms; (2) calibrates each partition separately using a known technique from the extensive line of calibration literature.

The heterogeneous partitioning can be done through a variety of tree-based methods, and we view this as a natural, efficient, and rigorous incorporation of tree-based techniques into DNNs through the use of calibration.
In fact, our theoretical optimality results also imply that heterogeneous calibration gives the optimal ensemble of a separately trained DNN and decision tree, combining the strengths of each into one model to maximize AUC.

Additionally the advantage of this post-hoc framework as opposed to applying techniques to fix over-confidence within the training itself is that overconfident models are not inherently undesirable with respect to accuracy~\cite{Guo2017}. 
The flexibility of over-parameterization allows the model training to simultaneously learn generalizable patterns and also memorize small portions of the training data. 
Validation data is often used to identify the point at which increased memorization outweighs the additional generalization, but decoupling these prior to this point and still achieving a similar level of performance is incredibly challenging. 
The post-hoc nature of our framework then allows us to avoid this difficulty and enjoy the additional generalization from over-confident models while also correcting the under-utilization of simpler patterns in the data.

\subsection{Notation}

To more rigorously set up the problem, we let $(\calX,\calY)$ be the data universe and we consider the classical binary classification setting where $\calX = \R^d$ and $\calY = \{0,1\}$ and  $(\xx,y) \in \R^d \times \{0,1\}$ is a feature vector and label from the data universe. 
Let $\calD$ be the probability distribution over the data universe with density function $d(\xx,y)$ where our data is $n$ random samples $\{(\xx_i,y_i)\}^n_{i=1} \stackrel{\textnormal{iid}}{\sim} \calD$.
Let $\calD_0$ and $\calD_1$ be the probability distributions over $\calX$ where we condition on the label being 0 and 1 respectively, which is to say that their respective density functions are such that $d_0(\xx) = \frac{d(\xx,0)}{\prob{(\xx,y) \sim \calD}{y=0}}$ and $d_1(\xx) = \frac{d(\xx,1)}{\prob{(\xx,y) \sim \calD}{y=1}}$.

Let $\s:\calX \rightarrow \R$ be the score function of a binary classification model. We consider this to be the output of the final neuron of the DNN prior to applying the sigmoid function, but our theoretical results hold for any score function.

We will be considering splits of the feature space where we let $\Pi = \{\Pi_i \}_{i=1}^m$ be a partitioning of $\calX$ such that each $\Pi_i \subseteq \calX$ where they cover $\calX$, which is to say $\bigcup_{i=1}^m \Pi_i = \calX$ and they are all disjoint so for any $i,j$ we have $\Pi_i \cap \Pi_j = \emptyset$.

We will also refer to \textit{heterogeneous partitions} in the feature space by which we most often mean that either $\prob{(\xx,y) \sim \calD}{y=1|\xx \in \Pi_i} >\!> \prob{(\xx,y) \sim \calD}{y=1}$ or $\prob{(\xx,y) \sim \calD}{y=1|\xx \in \Pi_i} <\!< \prob{(\xx,y) \sim \calD}{y=1}$.

For a more rigorous definition of \textit{over-confidence} we borrow the definitions of~\citet{Bai2021don}, where predicted probability for a given class is generally higher than the estimated probability.
This also leads to the notion of a \textit{well-calibrated model} whereby $s = \prob{(\xx,y) \sim \calD}{y=1 | \s(\xx) = s}$ for all $s$, and we give a more rigorous definition in Section~\ref{subsec:cal_def} for completeness.
For the most part, we will be considering \textit{post-hoc calibration} (which we often shorten to calibration) where a post-hoc transformation is applied to the classifier score function to achieve $t(s) = \prob{(\xx,y) \sim \calD}{y=1 | \s(\xx) = s}$ for all $s$.
Note that this cannot be equivalently defined as requiring $s = \prob{(\xx,y) \sim \calD}{y=1 | t(\s(\xx)) = s}$ because this could be perfectly achieved through setting $t(s) = \prob{(\xx,y) \sim \calD}{y=1}$ for all $s \in \R$ but lose all value of the classifier.

We focus our rigorous examination of accuracy with respect to the area under the curve (AUC) metric, which we precisely define here. 
Generally AUC is considered in terms of the receiver operating characteristic (ROC) curve, which is plotted based upon the True Positive Rate (TPR) vs False Positive Rate (FPR) at different thresholds.
This definition is known to be equivalent to randomly drawing a positive and negative label example and determining the probability that the model will identify the positive label. We also show this equivalence in the appendix for completeness. 

\begin{definition}\label{def:auc}[AUC]
For a given classifier score function $f: \calX \rightarrow \R$, along with distributions $\calD_0$ and $\calD_1$ then AUC can be defined as
\[
AUC(f,\calD_0,\calD_1) = \int_{\xx_0,\xx_1 \in \calX^2} d_0(\xx_0) d_1(\xx_1) \bigg( \one\{\s(\xx_1) > \s(\xx_0)\}  + \frac{1}{2}\one\{\s(\xx_1) = \s(\xx_0)\}  \bigg)
\]
\end{definition}

Note that we will often omit the relevant $dx$ terms for notational simplicity. We also give definitions of related metrics such as TPR, FPR, log-loss, Precision/Recall, and expected calibration error in the Appendix.





In this paper, we will first develop the intuition as to how \textit{over-confident} models tend to under-utilize \textit{heterogeneous partitions} of the feature space. Based on this intuition, the main focus of this paper, is to develop a framework that can leverage this heterogeneity to improve model generalization. Specifically, using a heterogeneous partition $\Pi$, how can we transform the score $s \rightarrow t(s)$ that optimizes AUC for binary classification tasks.

\section{Intuition for over-confident models under-utilizing heterogeneity}\label{sec:intuition} 

In this section we give intuition on why over-confidence due to over-parameterization can negatively impact model performance when there is heterogeneity in the data. 
Note that we will consider binary classification for ease of visualization, but the same ideas generalize to multi-classification where then the output score is a vector in $\R^k$.
We will set up this intuition by visualizing the distribution of scores for the positive and negative labels.
First we will give an example of what these distributions might look like on training vs test data and how they often differ due to over-parameterization.
Then we will consider independently adding a feature with heterogeneity and show how the over-confidence will lead to training data not properly accounting for that heterogeneity.

\subsection{Over-confident model example}

In order to visualize model performance it is common to look at the distributions of the score functions with respect to label. 
Specifically we want to empirically plot $\prob{\xx \sim \calD_0}{\s(\xx) = s}$ and $\prob{\xx \sim \calD_1}{\s(\xx) = s}$ for all $s \in \R$ which is often done by constructing a histogram of the scores with respect to their label. 
For our toy example, suppose our data is such that labels are balanced, so $\prob{(\xx,y) \sim \calD}{y=1} = \prob{(\xx,y)\sim \calD}{y=0}$. Further we will let $\calN(\cdot, \mu, \sigma)$ denote the Gaussian distribution with mean $\mu$ and variance $\sigma^2$.

Generally the over-parameterization of neural networks leads to training data performing significantly better than test data because the model performs some memorization of the training data.
Most often this memorization will occur on the harder data points to classify and better separate these examples compared to the test data.
Visually this tends to then lead to a steeper decline in the respective score distributions for the training data on the harder data points to classify.
Meanwhile for the test data the score distributions are much more symmetric because the model has not performed nearly as well on the hard-to-classify data points leading to more overlap.
An example visualization of over-confidence is in Figure \ref{fig:example}.

\begin{figure}[ht]
\centering
\includegraphics[width=0.45\textwidth]{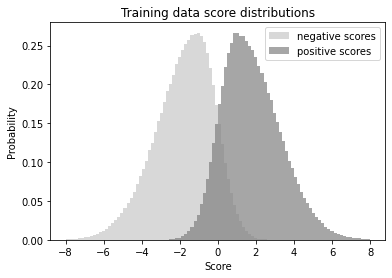}
\includegraphics[width=0.45\textwidth]{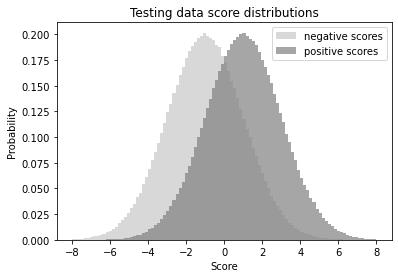}
\caption{Training and test score distributions for over-confident models\label{fig:example}}
\end{figure}

In this example we assume that our classifier function is such that $\prob{\xx \sim \calD_0}{\s(\xx) = s} = \calN(s;-1,{2})$ and $\prob{\xx \sim \calD_1}{\s(\xx) = s} = \calN(s;1,{2})$. Further let $\widehat{\calD}^{train} = \{(\xx_i,y_i)\}^{n}_{i=1} \stackrel{\textnormal{iid}}{\sim} \calD$ be the training data sample and let $\widehat{\calD}^{train}_0$ and $\widehat{\calD}^{train}_1$ be split into the negative and positive labels. In our example we set

\[
\prob{\xx \sim \widehat{\calD}^{train}_1}{\s(\xx) = s} = 
\begin{cases}
\frac{4}{3}\calN(s;1,2) &\text{if } s \geq 1
\\
\frac{2}{3}\calN(s;1,1) &\text{otherwise }
\end{cases}
\]

\[
\prob{\xx \sim \widehat{\calD}^{train}_0}{\s(\xx) = s} = 
\begin{cases}
\frac{4}{3}\calN(s;-1,2) &\text{if } s \leq -1
\\
\frac{2}{3}\calN(s;-1,1) &\text{otherwise }
\end{cases}
\]

This type of over-confidence on training data tends to be the root cause of mis-calibration.
The model does often inherently attempt to optimize calibration, for instance with a log-loss function, but it is doing so on the training data where it is over-confident in how well it has separated positive and negative labels and thus scales up the scores substantially, pushing the associated probabilities closer to 0 or 1.
In order to optimize the log-loss of the test data we would need to divide the score function by a factor of about 2, which would also give approximately optimal calibration. 
Note that the training data is also not optimized as we do assume some sort of regularization such as soft labels because the log loss would actually be optimized on the training data by scaling the score function up by a factor of about 2.
Regardless of how much it is scaled up or down, the ordering of the score function and all associated ordering metrics such as AUC or accuracy will remain unchanged.

\subsection{Under-utilized heterogeneity example}

While the over-confidence in our example above only affects the output probability and not the ordering, this over-confidence can be detrimental to ordering if we add heterogeneity to the data set.
Suppose we add a binary feature to our feature space that is uncorrelated with the other features but is well correlated with the label so it's heterogeneous.
Specifically, if we previously had $\calX = \R^d$ then we now consider $\calX' = \R^{d+1}$ and distribution $\calD'$ with density function $d': \calX' \times \calY \rightarrow [0,1]$ such that for any $(\xx',y)$ such that $\calD'(\xx',y) > 0$ we have $\xx'_{k+1} \in \{0,1\}$.
Further we assume that it's conditionally independent of the other features given the label, but it does well predict the label and in particular we have $\prob{(\xx',y) \sim \calD'}{y=1 | \xx'_{d+1} = 1} = \frac{3}{4}$ and $\prob{(\xx',y) \sim \calD'}{y=0 | \xx'_{d+1} = 0} = \frac{3}{4}$.


Assume that we use the same training and test data but with this new heterogeneous feature added to the dataset.
Due to the new feature being conditionally independent from the other features it's reasonable to assume that the score function the model would learn on the training set would (at least approximately) be $\s'(\xx') = \s(\xx'_{[1:d]}) + w_{d+1} \xx'_{d+1} + b$ for some optimized $w_{d+1}$ and $b$. 

The choice of $w_{d+1}$  determines the relative ordering of the score function when $\xx'_{d+1} = 1$ vs when $\xx'_{d+1} = 0$, so the extent to which the model should utilize the heterogeneity of this binary feature (and then $b$ simply re-centers the score function appropriately).
The better the model is performing the less it will need to use this additional heterogeneity to improve its prediction.
It is then important to note that this $w_{d+1}$ and $b$ are optimized on the training data where the model is over-confident in its performance and as such will not set $w_{d+1}$ nearly as high as it should for the true distribution.

In particular for the training data, it will set $w_{d+1} \approx 1.8$ and $b \approx -0.9$ to optimize cross-entropy which also maximizes AUC on the training data, and on the true data distribution this gives an AUC of about 0.83. 
Due to the over-confidence on the training data the model actually set this value lower than it should have and if instead it had set $w_{k+1} \approx 3.6$ and $b \approx -1.8$ then we could have increased the AUC to about 0.85 on the true data distribution and also improved the log-loss along with other accuracy metrics. 

\begin{figure}[h]
\centering
\includegraphics[width=0.48\textwidth]{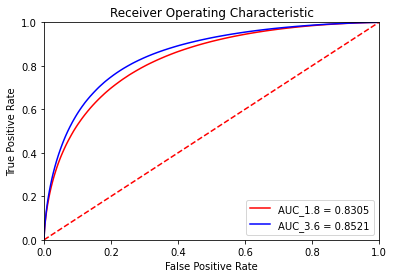}
\caption{Comparison of AUC with different $w_{d+1}$ \label{fig:ALL}}
\end{figure}




\subsection{General Discussion}

Our example above illustrates the more general concept of how neural networks can under-utilize simple patterns in the data because they are over-confident in their ability to fit 
the data.
This point is generally understood as a potential pitfall of neural networks but we are focusing specifically on how it fails to appropriately utilize heterogeneity.


In particular the bias term in the output layer can be viewed as a centering term for the score function 
to optimally account for the balance of positive vs negative labels.
This bias term will not affect the overall ordering, but the neural network can also make these centering decisions at a more fine-grained level where in our example above we considered simply splitting the data once. 
Especially if the internal nodes used a ReLU activation function then it would be quite simple for a neural network to construct internal variables that represent simple partitions of the data, reminiscent of partitions that are similarly defined by decision trees.
This could then lead to relative orderings between partitions that are inappropriate because the model centered the partitions according to the training data on which it was overconfident.




In our example we assumed that the new feature was conditionally independent and thus the appropriate fix was simply shifting each partition. 
With more intricate dependence we would expect the score distributions on each side of the split to differ more significantly than being identical up to a bias term.
Therefore ordering different partitions correctly with respect to the others will be a more complex task.
In Section~\ref{sec:auc_calibration} we show that the optimal way of ordering these partitions relative to each other is actually equivalent to optimally calibrating each partition.

\section{Calibration of partitions to optimize AUC}\label{sec:auc_calibration} 

In Section~\ref{sec:intuition} we provided intuition regarding over-confident models under-utilizing heterogeneity.
In this section we assume that such a heterogeneous partitioning has been identified and provide the theoretical framework for optimally applying a post-hoc transformation to maximize AUC.

\subsection{Optimal AUC calibration}

We first consider applying a post-hoc transformation to the classifier score function, in the same way as standard calibration and define the corresponding AUC measurement. 

\begin{definition}\label{def:cal_auc}[Calibrated AUC]
For a given classifier score function $\s: \calX \rightarrow \R$, along with distributions $\calD_0$ and $\calD_1$, and a transformation function $t: \R \rightarrow \R$ then we define calibrated AUC as

\[
AUC(\s,\calD_0,\calD_1,t) = \int_{s_0,s_1 \in \R^2} \prob{\xx \sim \calD_0}{\s(\xx) = s_0} \prob{\xx \sim \calD_1}{\s(\xx) = s_1} 
 \bigg( \one\{t(s_1) > t(s_0)\} + \frac{1}{2}\one\{t(s_1) = t(s_0)\}  \bigg)    
\]
\end{definition}

Note that when $t$ is the identity function or any isotonic function then this is equivalent to standard AUC. 
Further note that this is equivalent to $AUC(t(\s), \calD_0, \calD_1)$ but this notation will be easier to work with in our proofs for which we will give intuition here and prove in the appendix.

It is then natural to consider the optimal transformation function to maximize AUC conditioned on our classifier score function and data distribution.

\begin{lemma}[Informal]\label{lem:no_partition_likelihood}
Given a classifier score function $\s: \calX \rightarrow \R$ and any distributions $\calD_0,\calD_1$, we can maximize $AUC(\s,\calD_0,\calD_1,t)$ with respect to $t$ by using the likelihood ratio, $t(s) = \frac{\prob{\xx \sim \calD_1}{\s(\xx) = s}}{\prob{\xx \sim \calD_0}{\s(\xx) = s}}$ as our transformation function

\end{lemma}

For the purposes of maximizing AUC only the ordering imposed by $t$ is relevant, and intuitively the likelihood ratio will give the highest ordering to outputs that maximize True Positive Rate (TPR) and minimize False Positive Rate (FPR) thereby maximizing AUC.
Furthermore, we also show that for any FPR the corresponding TPR is maximized by the likelihood ratio transformation, which implies that the ROC curve of any other transformation is contained within the ROC curve of the likelihood ratio transformation.
As a corollary this implies that for any Recall the corresponding Precision is maximized and furthermore the PR-AUC is maximized by the likelihood ratio transformation.
While these claims are intuitively reasonable, they will require more involved proofs in the appendix.











We then show that the ordering from this likelihood ratio is equivalent to the ordering from the optimal calibration which by definition sets $t(s) = \prob{(\xx,y) \sim \calD}{y = 1 | \s(\xx) = s}$.

\begin{lemma}[Informal]\label{lem:no_partition_ordering}
The likelihood ratio and optimal calibration give an equivalent ordering. For any $s,s'$ we have that $\frac{\prob{\xx \sim \calD_1}{\s(\xx) = s}}{\prob{\xx \sim \calD_0}{\s(\xx) = s}} < \frac{\prob{\xx \sim \calD_1}{\s(\xx) = s'}}{\prob{\xx \sim \calD_0}{\s(\xx) = s'}}$ if and only if 
\[\prob{(\xx,y) \sim \calD}{y = 1 | \s(\xx) = s} < \prob{(\xx,y) \sim \calD}{y = 1 | \s(\xx) = s'}
\]
\end{lemma}


Due to the fact that AUC is invariant under equivalent orderings, calibration on the full dataset will also optimize AUC and other associated metrics.
This connection allows us to simply apply standard techniques from the literature for calibration to optimize AUC.
However, we expect this affect to be minimal even when the model is overconfident because it is more so correcting the over-confident probability estimations but not changing the ordering.

\subsection{Optimal partitioned AUC calibration}

While calibration on the full dataset may not generally affect ordering and thus AUC, recall that Section~\ref{sec:intuition} identified the issue of over-confidence negatively affecting the relative ordering between heterogeneous partitions of the data. 
In order to re-order these partitions appropriately, we then want to extend our optimal post-hoc transformation separately to each partition such that it provably maximizes overall AUC.



\begin{definition}\label{def:partition_cal_auc}[Partition Calibrated AUC]
For a given classifier score function $\s: \calX \rightarrow \R$, distributions $\calD_0,\calD_1$, and a partition $\Pi$ of $\calX$, along with a transformation function $t: \R \times \Pi \rightarrow \R$, we define partition calibrated AUC as
\begin{multline*}
AUC(\s, \calD_0,\calD_1,t, \Pi) = 
\\
\int_{s_0,s_1 \in \R^2} \int_{\Pi_i,\Pi_j \in \Pi^2} \prob{\xx \sim \calD_0}{\s(\xx) = s_0, \xx \in \Pi_i}  \prob{\xx \sim \calD_1}{\s(\xx) = s_1, \xx \in \Pi_j} 
 \bigg( \one\{t(s_1, \Pi_j) > t(s_0, \Pi_i)\}  + \frac{1}{2}\one\{t(s_1, \Pi_j) = t(s_0, \Pi_i)\}  \bigg)
\end{multline*}
\end{definition}

Once again if $t(s,\Pi_i) = s$ always then this is equivalent to $AUC(\s,\calD_0,\calD_1)$. 
Furthermore we could have equivalently defined this as $AUC(t(\s,\Pi),\calD_0,\calD_1)$ but this will be easier to work with in our proofs. 
For this definition we will also show that AUC is maximized by using the likelihood ratio.

\begin{lemma}[Informal]\label{lem:partition_likelihood}
Given classifier score function $\s: \calX \rightarrow \R$, and distributions $\calD_0$ and $\calD_1$, along with a partition $\Pi$ of $\calX$, we can maximize $AUC(\s,\calD_0,\calD_1,t,\Pi)$ by using the likelihood ratio $t(s, \Pi_i) = \frac{\prob{\xx \sim \calD_1}{\s(\xx) = s, \xx \in \Pi_i}}{\prob{\xx \sim \calD_0}{\s(\xx) = s, \xx \in \Pi_i}}$

\end{lemma}

Note that we could set $t(s,\Pi_i) = t(s,\Pi_j)$ for all $s,\Pi_i,\Pi_j$ and thus we can only improve (or keep equal) AUC by partitioning, and this holds for any arbitrary partition.
Additionally this likelihood ratio will give the same ordering as the optimal calibration for each partition which for a given $\Pi_i$ would set $t(s,\Pi_i) = \prob{(\xx,y) \sim \calD}{y = 1 | \s(\xx) = s,\xx \in \Pi_i}$.

\begin{lemma}[Informal]\label{lem:partition_ordering}
The likelihood ratio and the optimal calibration probability give an equivalent ordering. For any $s,s' \in \R$ and $\Pi_i,\Pi_j \in \Pi$ we have that $\frac{\prob{\xx \sim \calD_1}{\s(\xx) = s,\xx \in \Pi_i}}{\prob{\xx \sim \calD_0}{\s(\xx) = s,\xx \in \Pi_i}} < \frac{\prob{\xx \sim \calD_1}{\s(\xx) = s',\xx \in \Pi_j}}{\prob{\xx \sim \calD_0}{\s(\xx) = s',\xx \in \Pi_j}}$ if and only if 
\[
\prob{(\xx,y) \sim \calD}{y = 1 | \s(\xx) = s,\xx \in \Pi_i} < \prob{(\xx,y) \sim \calD}{y = 1 | \s(\xx) = s',\xx \in \Pi_j}    
\]
\end{lemma}

Therefore, by optimally calibrating each partition we can equivalently maximize overall AUC. If this partitioning is taken to the extreme then this calibration is
just intuitively the optimal model.

\begin{corollary}[Informal]

If $\Pi$ is the full partitioning of $\calX$ which is to say $|\Pi_i| = 1$ for all $i$, then the optimal $t(s,\Pi_i)$ is equivalent to $\prob{(\xx,y) \sim \calD}{y=1 | \xx }$ where $\Pi_i = \xx$ 

\end{corollary}

However running post-processing to accomplish the same task as the model training is both redundant and infeasible to be accurately done in this way.
It is then necessary to balance partitioning the feature space and still maintaining enough data to accurately calibrate each partition.
Furthermore, from the intuition we gave before, we aren't simply doing this partitioning in hopes of improvement because of the mathematical guarantee, but because the over-confidence means that our model may not have accounted for specific partitions appropriately.

The granularity to which the feature space can be partitioned and still maintain accuracy has also been studied in the fairness literature giving bounds on the sample complexity for multicalibration~\cite{shabat2020sample}, and there has also been work on estimating calibration of higher moments for multicalibration~\cite{pmlr-v134-jung21a}.
The sample complexity results are agnostic to calibration technique, but for a more practical application the extent of the partitioning should be dependent on the which calibration technique is applied.
For example, histogram binning essentially estimates the full score distributions and will require more samples to keep empirical error low.
In contrast, Platt scaling is just logistic regression on one variable and thus requires fewer samples to get accurate parameters for the calibration.

Additionally, the extent of the partitioning also depends upon whether we use a random forest for our partitioning scheme and take the average calibration over all the trees.
In the same way that the trees can have a greater depth because of the ensemble nature of a random forest, we could take advantage of the same ensemble type structure to partition more finely.
We could also apply tree pruning techniques via cross-validation \cite{kohavi1995study} to determine the ideal level of partitioning.










\section{Heterogeneous Calibration Framework}\label{sec:framework} 








Combining the intuition in Section~\ref{sec:intuition} whereby over-confident models under-utilize heterogeneous partitions, and the theoretical optimality in Section~\ref{sec:auc_calibration} of calibrating each partition separately to maximize AUC, immediately implies the general \textit{heterogeneous calibration} framework:


\begin{enumerate}[noitemsep]
    \item Partition the feature space for maximal heterogeneity
    
    \item Calibrate each partition of the feature space separately
    
    
\end{enumerate}

We give an explicit implementation of this framework in Section~\ref{subsec:simple}, but the flexibility of this paradigm allows for many possible implementations.
In particular, there are a multitude of post-hoc calibration techniques from the literature that could be applied ~\cite{platt1999probabilistic,zadrozny2002transforming,zadrozny2001obtaining,naeini2015obtaining,kumar2019verified,kull2019beyond}. 
Furthermore, splitting the feature space with a decision tree, which greedily maximizes heterogeneity, is the most obvious choice but we could also use 
a random forest here by repeating the partitioning and calibration multiple times and outputting the average across the trees.
We could also utilize boosted trees, which gives a sequence of partitions, and then sequentially apply calibration such that the final transformation was a nested composition of calibrations for each partitioning.
We further sketch out the details of how this could work for boosted trees in the appendix (Section~\ref{sec:boosted}), but leave a more thorough examination to future work.
Additionally, we could construct decision trees that greedily split the feature space to more directly optimize AUC that we discuss in Section~\ref{sec:algorithm}.




We note that
this framework can easily be applied to multi-class classification with many tree-based partitioning schemes and calibration techniques being extendable to multi-class classification.
We also focus upon tabular data and recommender systems because heterogeneity is much more common in these settings, but this framework could be extended to image classification and natural language processing. 
In particular the partitioning of the feature space could be identified by applying a decision tree to the neurons of an internal layer in the neural network, which are often considered to represent more general patterns and thus have more heterogeneity. 

In order to best show that the underlying theory of our general framework holds in practice, we focus on the simplest instantiation and leave the application of higher-performing tree-based partitioning schemes and more effective calibration techniques to future work.


\subsection{Example Implementation}\label{subsec:simple}


To exemplify our framework we give a simple instantiation here which will also be used in our experiments.
A decision tree classifier identifies the partitioning and logistic regression is used for calibration, which is Platt scaling.


We assume that the model is a DNN trained on the training data and the model with the highest accuracy on the validation is chosen, but this assumption is not necessary to apply this algorithm.
Our heterogeneous calibration procedure can use the same training and validation data.
However, by choosing the model with peak accuracy on validation data, it's likely the model is slightly over-confident on the validation data, although much less than the training data, and using fresh data for the calibration would be preferable.




\begin{algorithm}
\caption{Heterogeneous Calibration \label{algo:hetcal}}
\begin{algorithmic}[1]
\REQUIRE Training data $\widehat{\calD}^{train} = \{(\xx_i,y_i)\}^{n_t}_{i=1} \stackrel{\textnormal{iid}}{\sim} \calD$
, validation data
$\widehat{\calD}^{val} = \{(\xx_i,y_i)\}^{n_v}_{i=1} \stackrel{\textnormal{iid}}{\sim} \calD$
, and classifier score function $\s:\calX \rightarrow \R$ from a trained model
\STATE Build low depth classification tree on $\widehat{\calD}^{train}$ whose leaves' generate a partitioning of the feature space, $\Pi$
\FOR{Each partition $\Pi_i \in \Pi$}
\STATE Get label and score pairs for given partition $\calS_{\Pi_i}^{val} = \{(y,\s(\xx)) : (\xx,y) \in \widehat{\calD}^{val}, \xx \in \Pi_i  \}$
\STATE Run Platt scaling (logistic regression) on $\calS_{\Pi_i}^{val}$ to get calibration transformation $t_i$
\ENDFOR
\STATE For feature vector $\xx \in \calX$, use the classification tree to find partition $\Pi_i \subseteq \calX$ such that $\xx \in \Pi_i$ 
\STATE \textbf{Return} Probability prediction $t_i(\s(\xx))$
\end{algorithmic}
\end{algorithm}


This framework will be most effectively applied to real-world use cases under three general conditions: 
\begin{enumerate}[noitemsep]
    \item The model should have some degree of over-confidence in the same way that post-hoc calibration techniques give little additional value to well-calibrated models \item There should be an algorithmically identifiable partitioning of the feature space with a reasonable amount of heterogeneity 
    \item There should be sufficient data outside of the training data to accurately perform calibration on each partition
\end{enumerate}

\subsection{Interpolation between DNNs and tree-based algorithms}\label{subsec:interpolation}

In this section we further discuss how our heterogeneous calibration framework gives a natural interpolation between DNNs and tree-based algorithms through the use of calibration.
In particular, we show how this framework can equivalently be viewed as an optimal ensemble of any given DNN and decision tree through the use of calibration.
Furthermore, we discuss how this can extend to any tree-based algorithm.

We begin by re-considering Algorithm~\ref{algo:hetcal} whereby we could equivalently assume that we have learned a score classifier function $\s : \calX \rightarrow \R$ from a DNN, and also independently have learned a partitioning $\Pi$ through a decision tree classifier on the training data.
Therefore, we have two separate binary classification prediction models for a given feature vector $\xx \in \calX$.
Our DNN will give the probability prediction $\sigma(\s(\xx))$. 
Our decision tree classifier will identify the partition such that $\xx \in \Pi_i$ and return the probability prediction $\prob{(\xx, y) \sim \widehat{\calD}^{val}}{y=1| \xx \in \Pi_i}$.

Next we consider the logistic regression from Algorithm~\ref{algo:hetcal} which is done on each $\Pi_i$
and learns a function $\sigma(a_i x + b_i)$ over $a_i,b_i$.
Our heterogeneous calibration will combine the DNN and the partitioning from the decision tree, $\Pi$, such that for any feature vector $\xx \in \calX$ it will output the probability prediction $\sigma(a_i \s(\xx) + b_i)$ where $\xx \in \Pi_i$.
Note that if our logistic regression learns $a_i = 1$ and $b_i = 0$ for all partitions, then the new model is identical to the original DNN.
Similarly, if the logistic regression learns $a_i = 0$ and $b_i = \sigma^{-1}(\prob{(\xx,y) \sim \widehat{\calD}^{val}}{y = 1 | \xx \in \Pi_i})$ for all partitions, then this new model is equivalent to the original decision tree.
Accordingly, the calibration can then be seen as an interpolation between the DNN and decision tree model.

From our optimality results in Section~\ref{sec:auc_calibration}, we further know that perfect calibration will actually optimize the ensemble of these two models. 
Essentially, the calibration will implicitly pick and choose which strengths of each model to use in order to combine them to maximize AUC.
The natural interpolation of calibration between models equivalently extends to other tree-based algorithms such as random forests and boosted trees (further detail in Section~\ref{sec:boosted}). 
Extending the optimality to these settings should also follow similarly and we leave to future work.
Accordingly, our heterogeneous calibration framework can be equivalently viewed as way to optimally combine independently trained DNNs and tree-based algorithms in a post-hoc manner.
While this may theoretically guarantee an optimal combination, it's again important to note that the extent of partitioning and intricacy of calibration must be balanced with the corresponding empirical error for our framework to be effectively applied in practice.



\section{Experiments}\label{sec:experiments} 

\begin{table*}[h!]
\begin{tabular}{lllllll}
\hline
Size                        & Model                        & Bank Marketing  & Census data & Credit Default & Higgs & Diabetes \\ \hline
S  & Top 3 DNN               & 0.7758           & 0.8976           & 0.7784   & 0.7801     & 0.6915  \\
                            & Top 3 HC           &  0.7816 (+0.76\%)              &   0.9021 (+0.50\%)      &  0.7798 (+0.18\%)      &  0.7816 (+0.19\%)        &   0.6937 (+0.32\%)    \\ \cline{2-7} 
                            & Top 50\% DNN     &   0.7736               &        0.8892        &   0.7771     &    0.7650      &  0.6799     \\
                            & Top 50\% HC &   0.7810 (+0.96\%)             &    0.9004 (+1.27\%)     &   0.7789 (+0.23\%)     &   0.7692 (+0.54\%)       &  0.6879 (+1.18\%)     \\ \hline \hline
M & Top 3 DNN               & 0.7712           & 0.8978           & 0.7787    & 0.7773      & 0.6744   \\
                            & Top 3 HC           &  0.7800 (+1.14\%)     &         0.9027 (+0.55\%)      &  0.7794 (+0.09\%)       &  0.7799 (+0.33\%)        &  0.6856 (+1.66\%)    \\ \cline{2-7} 
                            & Top 50\% DNN    &   0.7690     &           0.8858     &  0.7775      &  0.7617        &   0.6683    \\
                            & Top 50\% HC &  0.7793 (+1.34\%)   &              0.9009 (+1.70\%)  &  0.7790 (+0.20\%)      &    0.7680 (+0.83\%)      &   0.6841 (+2.37\%)    \\ \hline \hline
L  & Top 3 DNN              & 0.7716           & 0.9007            & 0.7783    & 0.7747      & 0.6679   \\
                            & Top 3 HC           &    0.7814 (+1.27\%)            &    0.9027 (+0.22\%)            &  0.7794 (+0.14\%)      &   0.7775 (+0.36\%)       &  0.6824 (+2.17\%)    \\ \cline{2-7} 
                            & Top 50\% DNN     &  0.7663              &            0.8800    &  0.7772      &        0.7596  &   0.6637    \\
                            & Top 50\% HC & 0.7779 (+1.52\%)               &      0.9010 (+2.38\%)          & 0.7789 (+0.23\%)       & 0.7666 (+0.92\%)         &  0.6824 (+2.82\%)    \\ \hline \hline
\end{tabular}
\caption{Test AUC-ROC (mean of 5 runs) on different datasets before and after calibration. DNN = Deep neural network, HC = Heterogeneous calibration. We report model performance on the top 3 variants as well as the top 50\% variants for each model, where top 3 and top 50\% is determined by DNN performance prior to HC.}
\label{tab:main_results}
\end{table*}

\begin{table*}[!ht]
\begin{tabular}{lllllll}
\hline
                     & Model                        & Bank Marketing  & Census data & Credit Default & Higgs data  & Diabetes \\ \hline
  & Top 3 Reg DNN               & 0.7758           & 0.8976           & 0.7781   & 0.7801     & 0.6693  \\
                            & Top 3 Reg HC           &  0.7816 (+0.76\%)              &   0.9021 (+0.50\%)     &  0.7793 (+0.16\%)      &  0.7816 (+0.19\%)        &  0.6829 (+2.04\%)    \\ \hline 
                            & Top 3 Unreg DNN     &   0.7735               &        0.8773        &   0.7768     &    0.7498      &  0.6915     \\
                            & Top 3 Unreg HC &   0.7804 (+0.88\%)             &   0.8985 (+2.42\%)     &  0.7787 (+0.25\%)     &   0.7588 (+1.20\%)       &  0.6937 (+0.32\%)     \\ \hline \hline
\end{tabular}
\caption{Test effect of regularization on AUC-ROC (mean of 5 runs) on different datasets before and after calibration for the small MLPs. DNN = Deep neural network, HC = Heterogeneous calibration, Reg = regularized model, Unreg = unregularized model. Top 3 variants are chosen using the procedure mentioned in Table~\ref{tab:main_results}. Table~\ref{tab:reg_results_appendix} in Appendix~\ref{sec:experiments_app} contains more results for medium and large networks.}
\label{tab:reg_results}
\end{table*}

We evaluate the efficacy of heterogeneous calibration on the task of binary classification using deep neural networks on a variety of datasets. We make observations about the effect of model size, regularization and training set size on the effectiveness of the technique. All experiments were conducted using Tensorflow~\cite{tensorflow2015-whitepaper}.

\textbf{Datasets:} We use datasets containing a varying number of data points and types of features. For each dataset, we create training, validation (for tuning neural networks), calibration (for training post-hoc calibration models) and test splits. Specifically, we use the following 5 datasets:
\begin{itemize}
    \itemsep-0.1em 
    \item Bank marketing~\cite{moro2014data} -  Marketing campaign data to predict client subscriptions. $\sim 45k$ datapoints.
    \item Census Income~\cite{kohavi1996scaling} - Data to predict whether income exceeds a threshold or not. $\sim 49k$ datapoints.
    \item Credit Default~\cite{yeh2009comparisons} - Data to predict credit card default. $\sim 30k$ datapoints.
    \item Higgs~\cite{baldi2014searching} - Data to distinguish between a Higgs boson producing signal process and a background process. We chose $\sim 50k$ datapoints out of the entire set.
    \item Diabetes~\cite{strack2014impact} - Data about readmission outcomes for diabetic patients. $\sim 100k$ datapoints.
\end{itemize}

Further details about the datasets, including features, splits and pre-processing information, can be found in Appendix~\ref{app:datasets}.

\textbf{Modeling details:} We use multilayer perceptrons with 3 feed-forward layers. To understand the effect of model size and model regularization on calibration performance, we vary the number of neurons in each MLP layer and also toggle regularization techniques like batch normalization~\cite{ioffe2015batch} and dropout~\cite{srivastava2014dropout}. Specifically, we choose 3 MLP sizes based on the number of parameters in each. We use the Adam optimizer~\cite{kingma2014adam} and extensively tune the learning rate on a log-scaled grid for each variant, since even adaptive optimizers can benefit from learning rate tuning~\cite{loshchilov2017decoupled}. Complete details about the MLP variants, regularization techniques and tuning of the learning rate can be found in Appendix~\ref{app:modelhp}. 

For heterogeneous calibration, we train a decision tree classifier on the training set to partition the feature space and subsequently use Platt scaling~\cite{platt1999probabilistic} on the calibration dataset for each partition. We lightly tune the tree hyperparameters, details in Appendix~\ref{app:dt_and_platt_scaling}. Note that extensive tuning or using a different partitioning or calibration algorithm could have led to further improvements for our method. 

\subsection{Main results}

Table~\ref{tab:main_results} displays the main results. We choose 3 MLP sizes based on the number of parameters and label them small, medium and large, based on the number of parameters. For each MLP size, we choose the top 3 and top 50\% variants after extensive tuning of learning rate and regularization and report the mean of 5 runs.

We note that our method provides a consistent lift in AUC across all model sizes and datasets, despite the use of a simple calibration model. This meshes well with our hypothesis that modern neural networks, despite regularization, are overconfident and calibration can be used as a simple post-hoc technique to improve generalization performance.

\begin{figure}[ht]
\centering
\includegraphics[width=0.47\textwidth]{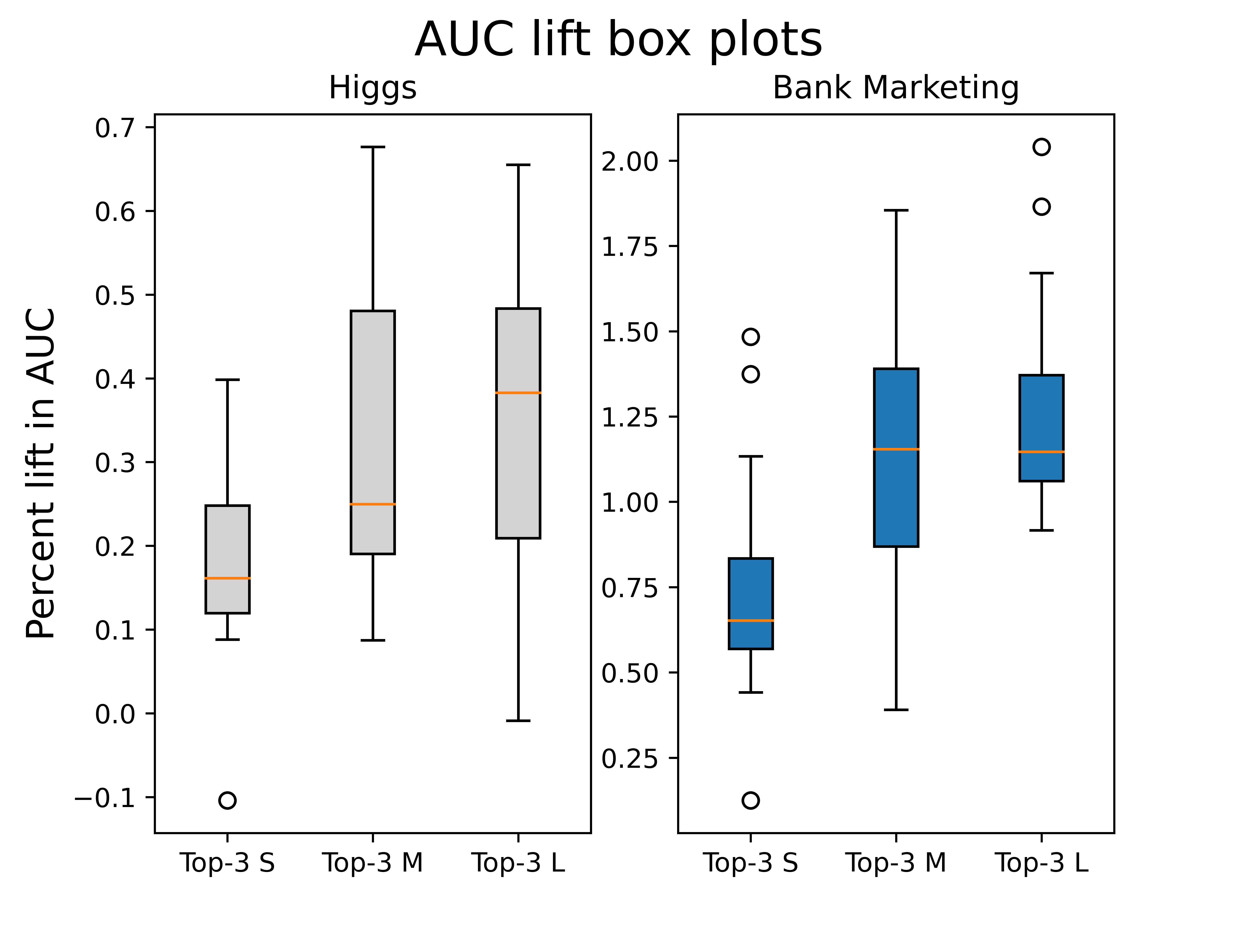}
\vspace{-5mm}
\caption{Box plots of test AUC lifts of our method for various runs of the top 3 models. We note a consistent lift in AUC across runs and hyperparameter settings.}
\label{fig:main_box_plots}
\end{figure}

Figure~\ref{fig:main_box_plots} contains box plots of the test AUC lift provided by our method for 2 datasets. The plots contain lifts 5 different runs of the top 3 models for each setting. We observe a consistent lift in AUC across various runs and hyperparameter settings, demonstrating the consistency of our method. We include box plots for other datasets in Appendix~\ref{app:more_results}.

\textbf{Effect of model size.} From Table~\ref{tab:main_results}, we note that as we go to larger models, the lift in performance for our method consistently increases for all datasets. This corroborates our hypothesis and intuition that larger models can be more overconfident, and hence may benefit more from our method.

\subsection{Effect of model regularization} Table~\ref{tab:reg_results} contains the results of the effect of heteregeneous calibration on regularized (use of dropout or batch normalization) and unregularized MLPs of small size. Unsurprisingly, our method provides larger relative lift in performance for unregularized DNNs as compared to regularized DNNs. This fits well with our hypothesis that unregularized networks are highly overconfident, and may benefit from methods such as ours.

\subsection{Computational efficiency}

We note that hyperparameter tuning is critical for improving generalization performance. Model performance varies widely with the tuning of hyperparameters. For our experiments, we tuned the learning rate. Interestingly, we note that our method has a much tighter variance for AUC across a large range of the learning rate, when compared to an uncalibrated network. This was particularly notable for the Census data where our technique maintained high performance even when the uncalibrated network performance dipped. 
This may reduce the need for extensive hyperparameter tuning.

\section{Discussions}\label{sec:discussion} 


In this paper we developed the framework of heterogeneous calibration that utilizes data heterogeneity and post-hoc calibration techniques for improving model generalization for over-confident models. We theoretically proved  that the calibration transformation is optimal in improving AUC. 
To show its efficacy in practice, we focus on the simplest instantiation, but this framework can naturally apply combinations of known higher-performing techniques for both the partitioning and calibration.
We believe further investigation into these applications of the framework would be an interesting and fruitful future direction now that we have established the efficacy of our heterogeneous calibration paradigm.

We further showed that our framework equivalently uses calibration to optimally combine a DNN and decision tree as a post-hoc ensemble method. 
This should extend to other tree-based algorithms in the same manner, but a more rigorous examination would be an interesting future direction. 
This investigation could also include a more thorough characterization of when the AUC increases most for this optimal combination of DNNs and tree-based algorithms.
This would potentially be used in determining how to train DNNs to focus on learning patterns that are not identifiable through tree-based algorithms and then utilize the heterogeneous calibration framework to achieve a higher-performing combination. 



Our experiments also showed much more consistent high performance of the model with heterogeneous calibration applied as we searched through the hyper-parameters.
We think another interesting future direction would be to further investigate the extent to which heterogeneous calibration can serve as a replacement for hyper-parameter tuning. 

\section{Acknowledgements}

We thank our colleagues Joojay Huyn, Varun Mithal, Preetam Nandy, Jun Shi, and Ye Tu for their helpful feedback and illuminating discussions.

\vspace{-0.30cm}
\bibliography{references}

\section{Partition Calibrated AUC Proofs}\label{sec:auc_proofs} 

In this section we provide the missing proofs of the informal Lemmas~\ref{lem:no_partition_likelihood},~\ref{lem:no_partition_ordering},~\ref{lem:partition_likelihood}, and~\ref{lem:partition_ordering}. Note that Lemmas~\ref{lem:partition_likelihood} and~\ref{lem:partition_ordering} are the more general case of the former two respectively, so we will only prove each of these.


We will copy the definition of partition calibrated AUC here for reference.

\begin{definition}\label{def:partition_cal_auc2}[Partition Calibrated AUC]
For a given classifier score function $\s: \calX \rightarrow \R$, and distributions $\calD_0$ and $\calD_1$, along with a partition $\Pi$ of $\calX$, and a transformation function $t: \R \times \Pi \rightarrow \R$, we define partition calibrated AUC as

\begin{multline*}
AUC(\s, \calD_0,\calD_1,t, \Pi) = 
\\
\int_{s_0,s_1 \in \R^2} \int_{\Pi_i,\Pi_j \in \Pi^2} \prob{\xx \sim \calD_0}{\s(\xx) = s_0, \xx \in \Pi_i}  \prob{\xx \sim \calD_1}{\s(\xx) = s_1, \xx \in \Pi_j} 
 \bigg( \one\{t(s_1, \Pi_j) > t(s_0, \Pi_i)\}  + \frac{1}{2}\one\{t(s_1, \Pi_j) = t(s_0, \Pi_i)\}  \bigg)
\end{multline*}

\end{definition}

Our formalized version Lemma~\ref{lem:partition_likelihood} can then be stated as such. 

\begin{lemma}\label{lem:opt_transform_partition}
For a given classifier score function $\s: \calX \rightarrow \R$, and distributions $\calD_0$ and $\calD_1$, along with a partition $\Pi$ of $\calX$, let $t^*:\R \times \Pi \rightarrow [0,1]$ be the transformation 

\[
t^*(s, \Pi_i) = \frac{\prob{\xx \sim \calD_1}{\s(\xx) = s, \xx \in \Pi_i}}{\prob{\xx \sim \calD_1}{\s(\xx) = s, \xx \in \Pi_i} + \prob{\xx \sim \calD_0}{\s(\xx) = s, \xx \in \Pi_i}} 
\]

and if ${\prob{\xx \sim \calD_1}{\s(\xx) = s, \xx \in \Pi_i} + \prob{\xx \sim \calD_0}{\s(\xx) = s, \xx \in \Pi_i}} = 0$ then we let $t^*(s, \Pi_i) = 0$.

For any function $t: \R \times \Pi \rightarrow \R$ we have $AUC(\s, \calD_0,\calD_1,t^*, \Pi) \geq AUC(\s, \calD_0,\calD_1,t, \Pi)$

\end{lemma}

Note that we slightly deviate from the likelihood ratio in Lemma~\ref{lem:partition_likelihood} to avoid divide by zero concerns but the optimal transformation in the lemma gives essentially an equivalent ordering to the likelihood ratio.

The notation will become too onerous here so we will let $p_0(s_0,\Pi_i) \defeq \prob{\xx \sim \calD_0}{\s(\xx) = s_0, \xx \in \Pi_i}$ and similarly for other subscripts. We will further let $T((s_1,\Pi_j), (s_0,\Pi_i)) \defeq \one\{t(s_1, \Pi_j) > t(s_0, \Pi_i)\} + \frac{1}{2}\one\{t(s_1, \Pi_j) = t(s_0, \Pi_i)\} $

\begin{proof}

We re-arrange the integrals to pair together an $s$ and $\Pi$ which allows us to re-write the AUC as 

\[
\int_{s_0,\Pi_i \in \R \times \Pi} \int_{s_1,\Pi_j \in \R \times \Pi} p_0(s_0,\Pi_i) p_1(s_1,\Pi_j) 
\cdot T((s_1,\Pi_j), (s_0,\Pi_i))
\]

We could then further consider each pair $(s_0,\Pi_i)$ and $(s_1,\Pi_j)$ and re-write the AUC as

\begin{multline*}
\int_{s,\Pi_i \in \R \times \Pi} p_0(s,\Pi_i) p_1(s,\Pi_i) T((s,\Pi_i), (s,\Pi_i))
\\
+ \int_{(s_0,\Pi_i) \neq (s_1, \Pi_j) \in (\R \times \Pi)^2} 
 p_0(s_0,\Pi_i) p_1(s_1,\Pi_j) T((s_1,\Pi_j), (s_0,\Pi_i)) 
 + p_0(s_1,\Pi_j) p_1(s_0,\Pi_i) T((s_0,\Pi_i), (s_1,\Pi_j))
\end{multline*}

The first integral does not change regardless of the choice of $t$. 
From Lemma~\ref{lem:cal_auc_order} we have that the inequality is tight in Lemma~\ref{lem:cal_auc_max} for all $\Pi_i,\Pi_j$ and $s_0,s_1$ when $t \equiv t^*$, and therefore the second integral is maximized with $t^*$.

\end{proof}

We utilize the following helper lemma that consider two pairs of scores and partition and gives an upper bound on their sum for both possible orderings.

\begin{lemma}\label{lem:cal_auc_max}

For any $\Pi_i,\Pi_j$ and $s_0,s_1$ we have 

\begin{multline*}
p_0(s_0,\Pi_i) p_1(s_1,\Pi_j) T((s_1,\Pi_j), (s_0,\Pi_i))  +
p_0(s_1,\Pi_j) p_1(s_0,\Pi_i) T((s_0,\Pi_i), (s_1,\Pi_j)) 
\\
 \leq \max\{  p_0(s_0,\Pi_i) p_1(s_1,\Pi_j) ,  p_0(s_1,\Pi_j) p_1(s_0,\Pi_i) \}
\end{multline*}

\end{lemma}

\begin{proof}
This follows immediately from the fact that 

\[
 T((s_1,\Pi_j), (s_0,\Pi_i))  +  T((s_0,\Pi_i), (s_1,\Pi_j))   = 1
\]

and both terms are non-negative.

\end{proof}

We also utilize another helper lemma that shows an equivalent ordering for our considered optimal transformation function with respect to pairs of scores and partitions.

\begin{lemma}\label{lem:cal_auc_order}
For any $\Pi_i,\Pi_j$ and $s_0,s_1$, if $  p_0(s_0,\Pi_i) p_1(s_1,\Pi_j)  <  p_0(s_1,\Pi_j) p_1(s_0,\Pi_i) $ then 

\[
\frac{p_1(s_1,\Pi_j)}{p_1(s_1,\Pi_j) + p_0(s_0,\Pi_j)} < \frac{p_1(s_0,\Pi_i)}{p_1(s_0,\Pi_i) + p_0(s_1,\Pi_i)} 
\]

\end{lemma}

\begin{proof}
If $  p_0(s_0,\Pi_i) p_1(s_1,\Pi_j)  <  p_0(s_1,\Pi_j) p_1(s_0,\Pi_i) $ then we must have $p_0(s_0,\Pi_i) + p_1(s_0,\Pi_i) > 0$ and $p_0(s_1,\Pi_j) + p_1(s_1,\Pi_j) > 0$ because otherwise 
$p_0(s_1,\Pi_j) p_1(s_0,\Pi_i) = 0$ contradicting our assumed inequality.

By adding $p_1(s_0,\Pi_i) p_1(s_1,\Pi_j)$ to both sides, our assumed inequality can then be equivalently written

\[
\bigg(p_0(s_0,\Pi_i) + p_1(s_0,\Pi_i) \bigg) p_1(s_1,\Pi_j) <
\bigg(p_0(s_1,\Pi_j) + p_1(s_1,\Pi_j) \bigg) p_1(s_0,\Pi_i)
\]

and dividing each side gives the desired inequality.

\end{proof}




\subsection{Ordering equivalence}

We further show that the ordering of the optimal transformation is equivalent to the ordering given where each partition is perfectly calibrated.

We keep the notation for the lemma statement the same but will switch to shorthand for the proof where we let $p_0(\cdot) \defeq \prob{\xx \sim \calD_0}{\cdot}$, $p_1(\cdot) \defeq \prob{\xx \sim \calD_1}{\cdot}$, and $p(\cdot) \defeq \prob{(\xx,y) \sim \calD}{\cdot}$. 

\begin{lemma}\label{lem:ordering_equivalence}
Given distributions $\calD, \calD_0, \calD_1$ and our classifier score function $\s: \calX \rightarrow \R$ and a partition $\Pi$ of $\calX$. For any $s,s' \in \R$ and $\Pi_i,\Pi_j \in \Pi$ where $\prob{(\xx,y) \sim \calD}{\s(\xx) = s,\xx \in \Pi_i} > 0$ and $\prob{(\xx,y) \sim \calD}{\s(\xx) = s', \xx \in \Pi_j} > 0$, then 

\[
   \prob{(\xx,y) \sim \calD}{y=1 |  \s(\xx) = s', \xx \in \Pi_j} <
    \prob{(\xx,y) \sim \calD}{y=1 |  \s(\xx) = s, \xx \in \Pi_i} 
\]

if and only if 

\[
\frac{\prob{\xx \sim \calD_1}{\s(\xx) = s', \xx \in \Pi_j}}{\prob{\xx \sim \calD_1}{\s(\xx) = s', \xx \in \Pi_j}+\prob{\xx \sim \calD_0}{\s(\xx) = s', \xx \in \Pi_j}} 
< \frac{\prob{\xx \sim \calD_1}{\s(\xx) = s, \xx \in \Pi_i}}{\prob{\xx \sim \calD_1}{\s(\xx) = s, \xx \in \Pi_i}+\prob{\xx \sim \calD_0}{\s(\xx) = s, \xx \in \Pi_i}} 
\]

\end{lemma}

\begin{proof}
By the definition of conditional probability we have 

\[
p(y=1 |  \s(\xx) = s, \xx \in \Pi_i) 
= 
\frac{p(y=1,\s(\xx) = s, \xx \in \Pi_i)}{p(y=1,\s(\xx) = s, \xx \in \Pi_i) + p(y=0,\s(\xx) = s, \xx \in \Pi_i)}
\]

Plugging this in to the first inequality in our if and only if statement, we then cross multiply and cancel like terms to get 

\begin{multline*}
p(y=1 |  \s(\xx) = s', \xx \in \Pi_j) < 
p(y=1 |  \s(\xx) = s, \xx \in \Pi_i)
\\
\Longleftrightarrow
\\ 
p(y = 1, \s(\xx) = s', \xx \in \Pi_j)p(y = 0, \s(\xx) = s, \xx \in \Pi_i) <
\\
p(y=1, \s(\xx) = s, \xx \in \Pi_i)p(y=0, \s(\xx) = s', \xx \in \Pi_j)
\end{multline*}

By our definitions we have $p(\s(\xx) = s, \xx \in \Pi_i | y=1) = p_1(\s(\xx) = s, \xx \in \Pi_i) $ which then implies $p(y=1,\s(\xx) = s, \xx \in \Pi_i) = p(y=1) p_1({\s(\xx) = s, \xx \in \Pi_i})$ and applying this and cancelling gives

\begin{align*}
p({y=1 |  \s(\xx) = s', \xx \in \Pi_j}) &< 
p({y=1 |  \s(\xx) = s, \xx \in \Pi_i})
\\& 
\Longleftrightarrow
\\ 
p_1({\s(\xx) = s', \xx \in \Pi_j})p_0({\s(\xx) = s, \xx \in \Pi_i}) &<
p_1({\s(\xx) = s, \xx \in \Pi_i})p_0({\s(\xx) = s', \xx \in \Pi_j})
\end{align*}

Furthermore by taking the second inequality in our desired if and only if statement, then cross multiplying and cancelling like terms we equivalently get

\begin{align*}
\frac{p_1({\s(\xx) = s', \xx \in \Pi_j})}{p_1({\s(\xx) = s', \xx \in \Pi_j}) + p_0({\s(\xx) = s', \xx \in \Pi_j})} &< 
\frac{p_1({\s(\xx) = s, \xx \in \Pi_i})}{p_1({\s(\xx) = s, \xx \in \Pi_i})+p_1({\s(\xx) = s, \xx \in \Pi_i})} 
\\& 
\Longleftrightarrow
\\
p_1({\s(\xx) = s', \xx \in \Pi_j})p_0({\s(\xx) = s, \xx \in \Pi_i}) &<
p_1({\s(\xx) = s, \xx \in \Pi_i})p_0({\s(\xx) = s',  \xx \in \Pi_j})
\end{align*}

\end{proof}

\section{Calibrated FPR and TPR Proofs}\label{sec:fpr_tpr} 

In this section we show that the same transformation function that optimizes calibrated AUC will also optimize TPR with respect to FPR. 
In particular, we show that the ROC curve for the optimal transformation function will always contain the ROC curve for any other transformation function.
As a corollary we obtain that the Precision with respect to Recall is also optimized and thus the PR-AUC is maximized.
We begin by defining calibrated TPR and FPR.

\begin{definition}\label{def:calibrated_tpr}[Calibrated TPR]
For a given classifier score function $\s: \calX \rightarrow \R$, and distribution $\calD_1$, along with a transformation function $t: \R \rightarrow \R$, and some $T \in \R$ and $q \in [0,1]$, we define calibrated TPR as

\[
    TPR_t(T,q) = 
    \int_{s \in \R}  \prob{\xx \sim \calD_1}{\s(\xx) = s} (\one\{t(s) > T\} + q \one\{t(s) = T\})
\]

\end{definition}

\begin{definition}\label{def:calibrated_fpr}[Calibrated FPR]
Defined identically to TPR but using distribution $\calD_0$

\end{definition}

The value of $q$ is necessary here because if $\prob{\xx \sim \calD_1}{\s(\xx) = s}$ is not continuous over $\R$ then $FPR^{-1}(x)$ may not be defined for all $x \in (0,1)$ and we want our statements to generalize over all probability distributions and score classifier functions.
Note that when $t$ is the identity function then this is just the standard definition for TPR and FPR.
As is well-known we can equivalently define AUC using TPR and FPR and we will prove the calibrated version here as well that follows equivalently.

\begin{lemma}
For a given classifier score function $\s: \calX \rightarrow \R$, along with distributions $\calD_0$ and $\calD_1$, and a transformation function $t: \R \rightarrow \R$ 

\[
AUC(\s,\calD_0,\calD_1,t) = \int_0^1 TPR_t(FPR_t^{-1}(x)) dx
\]

\end{lemma}

\begin{proof}

\begin{align*}
AUC(\s,\calD_0,\calD_1,t) & = 
\int_{s_0,s_1 \in \R^2} \prob{\xx \sim \calD_0}{\s(\xx) = s_0} \prob{\xx \sim \calD_1}{\s(\xx) = s_1} 
\left( \one\{t(s_1) > t(s_0)\} + \frac{1}{2}\one\{t(s_1) = t(s_0)\}  \right)
\\&
= \int_{s_0,s_1 \in \R^2} \prob{\xx \sim \calD_0}{\s(\xx) = s_0} \prob{\xx \sim \calD_1}{\s(\xx) = s_1} 
 \int_0^1 \left( \one\{t(s_1) > t(s_0)\} + q\one\{t(s_1) = t(s_0)\}  \right) dq
\\&
= \int_0^1  \int_{s_0 \in \R} \prob{\xx \sim \calD_0}{\s(\xx) = s_0}
 \int_{s_1 \in \R} \prob{\xx \sim \calD_1}{\s(\xx) = s_1} 
 \big( \one\{t(s_1) > t(s_0)\} + q\one\{t(s_1) = t(s_0)\}  \big) dq
\\&
= \int_0^1  \int_{s_0 \in \R} \prob{\xx \sim \calD_0}{\s(\xx) = s_0} TPR_t(t(s_0),q) dq
\\&
= \int_0^1  \int_{-\infty}^{\infty} \prob{\xx \sim \calD_0}{t(\s(\xx)) = T} TPR_t(T,q) dT dq
\\& 
= \int_0^1 TPR_t(FPR_t^{-1}(x)) dx
\end{align*}

\end{proof}

We then show that the transformation used to optimize AUC will also maximize TPR with respect to any given FPR, and the equivalence in the ordering from Lemma~\ref{lem:ordering_equivalence} implies then that perfectly calibrating the output probability estimates will also have this same property.

Recall that the general intuition here was that our optimal transformation will give the highest ordering to outputs that maximize the probability of being drawn from $\calD_1$ vs $\calD_0$ which should then maximize TPR and minimize FPR.
The proof here will be more involved, primarily due to the analysis of integrals ruling out analyzing through a greedy ordering, but will still apply this underlying intuition.

\begin{lemma}\label{lem:fpr_tpr_main}

For any $x \in [0,1]$ we have $TPR_{t^*}(FPR_{t^*}^{-1}(x)) \geq TPR_t(FPR_t^{-1}(x))$ where $t^*$ is defined as

\[
t^*(s) \defeq 
\frac{\prob{\xx \sim \calD_1}{\s(\xx) = s}}{\prob{\xx \sim \calD_1}{\s(\xx) = s} + \prob{\xx \sim \calD_0}{\s(\xx) = s}} 
\] 

and if ${\prob{\xx \sim \calD_1}{\s(\xx) = s} + \prob{\xx \sim \calD_0}{\s(\xx) = s}} = 0$ then let $t^*(s) = 0$.

\end{lemma}

\begin{proof}

For notational simplicity we will let $p_0(s) \defeq \prob{\xx \sim \calD_0}{\s(\xx) = s} $ and $p_1(s) \defeq \prob{\xx \sim \calD_1}{\s(\xx) = s} $. 
In order to avoid divide by zero concerns we will also let $\calX_{+} \defeq \{ s \in \R: p_0(s) + p_1(s) > 0\}$ which is just the union of the support of each.
 For a given $t: \R \rightarrow \R$ and some $T \in \R$ and $q \in [0,1]$, we define $\calS_t^{>} = \{ s \in \calX_{+}: t(s) > T, p_1(s) > 0 \}$ and $\calS_t^{=} = \{ s \in \calX_{+}: t(s) = T, p_1(s) > 0 \}$. We can then re-write the definition of calibrated TPR as 

\[
TPR_t(T,q) = \int_{s \in \calS_t^{>}}  p_1(s) + q \int_{s \in \calS_t^{=}} p_1(s)
\]

and equivalently for $FPR_t(T,q)$ with $\calD_0$. 
Given our value of $x \in [0,1]$, suppose we have some function $t: \R \rightarrow \R$ such that $TPR_{t^*}(FPR_{t^*}^{-1}(x)) < TPR_t(FPR_t^{-1}(x))$ and we can further assume that for any $t': \R \rightarrow \R$ we have $TPR_{t'}(FPR_{t'}^{-1}(x)) \leq TPR_t(FPR_t^{-1}(x))$ because there must be some such optimal function for $x$.  In order to prove our claim it then suffices to show  a contradiction on either of these inequalities.

We will let $T \in \R$ and $q \in [0,1]$ be such that $FPR_t^{-1}(x) = (T,q)$ and $T^* \in \R$ and $q^* \in [0,1]$ be such that $FPR_{t^*}^{-1}(x) = (T^*,q^*)$. Accordingly we will again define let $\calS_t^{>} = \{ s \in \calX_{+}: t(s) > T \}$ and $\calS_t^{=} = \{ s \in \calX_{+}: t(s) = T \}$ and $\calS_{t^*}^{>} = \{ s \in \calX_{+}: t^*(s) > T^* \}$ and $\calS_{t^*}^{=} = \{ s \in \calX_{+}: t^*(s) = T^*\}$. 

We will first consider the case of  $\calS_t^{=} \cap \calS_{t^*}^{>} \neq \emptyset$ and $q \in (0,1)$ and show a contradiction on the optimality of $t$ for the given $x$.  We consider two cases, if $\calS_t^{=} \cap \calS_{t^*}^{>} = \calS_t^{=}$ and $\calS_t^{=} \cap \calS_{t^*}^{>} \subset \calS_t^{=}$.

If $\calS_t^{=} \cap \calS_{t^*}^{>} \subset \calS_t^{=}$ then let $\calS_{+} \defeq \calS_t^{=} \cap \calS_{t^*}^{>}$ and $\calS_{-} \defeq \calS_t^{=} \setminus \calS_{t^*}^{>}$ and by construction of $t^*$ we must then have $\frac{p_1(s)}{p_1(s) + p_0(s)} > \frac{p_1(s')}{p_1(s') + p_0(s')}$ for any $s \in \calS_{+}$ and $s' \in \calS_{-}$.
From Lemma~\ref{lem:fpr_tpr_simple_helper} there must then be $q_1 > q_2$ such that $q_1 \int_{s \in \calS_{+}} p_0(x) + q_2 \int_{s \in \calS_{-}} p_0(s) = q ( \int_{s \in \calS_{+}} p_0(s) + \int_{s \in \calS_{-}} p_0(x)) $ such that either $q_1 = 1$ and $q_2 \geq 0$ or $q_2 = 0$ and $q_1 \leq 1$. In the case of $q_1 = 1$ we will define $t' : \R \rightarrow \R$ to be equivalent to $t$ except that $t'(s) > T$ for any $s \in \calS_{+}$. We then have $FPR_{t'}^{-1}(x) = (T,q_2)$ and by Lemma~\ref{lem:fpr_tpr_helper} we have $TPR_{t'}(T,q_2) > TPR_t(T,q)$ giving a contradiction on the optimality of $t$. Similarly if $q_2 = 0$ we will define $t' : \R \rightarrow \R$ to be equivalent to $t$ except that $t'(s) < T$ for any $s \in \calS_{-}$. We then have $FPR_{t'}^{-1}(x) = (T,q_1)$ and by Lemma~\ref{lem:fpr_tpr_helper} we have $TPR_{t'}(T,q_2) > TPR_t(T,q)$ giving a contradiction on the optimality of $t$.

If $\calS_t^{=} \cap \calS_{t^*}^{>} = \calS_t^{=}$, then define $\calS_{-} \defeq \calS_t^{>} \setminus \calS_{t^*}^{>}$ and again by construction of $t^*$ we must then have $\frac{p_1(s)}{p_1(s) + p_0(s)} > \frac{p_1(s')}{p_1(s') + p_0(s')}$ for any $s \in \calS_{+}$ and $s' \in \calS_{-}$.
Furthermore $\calS_{-} \neq \emptyset$ because otherwise we would have $FPR_t(T,q) < \int_{s \in \calS_{t^*}^{>}} p_0(s) \leq FPR_{t^*}(T^*,q^*) = x$ because we assumed $q<1$ and this contradicts $FPR_t^{-1}(x) = (T,q)$. 
From Lemma~\ref{lem:fpr_tpr_simple_helper2} there must then be some $q' \in (q,1)$ such that $q \int_{s \in \calS_{+}} p_0(x) + \int_{s \in \calS_{-}} p_0(s) = q' ( \int_{s \in \calS_{+}} p_0(s) + \int_{s \in \calS_{-}} p_0(x)) $. 
Once again we define $t' : \R \rightarrow \R$ to be equivalent to $t$ except that $t'(s) = T$ for any $s \in \calS_{-}$. 
We then have $FPR_{t'}^{-1}(x) = (T,q')$ and by Lemma~\ref{lem:fpr_tpr_helper} we have $TPR_{t'}(T,q') > TPR_t(T,q)$ giving a contradiction on the optimality of $t$.

We then consider the case of $\calS_t^{=} \cap \calS_{t^*}^{>} = \emptyset$ or $q \in \{0,1\}$ and we will show that $TPR_{t^*}(T^*,q^*) \geq TPR_{t}(T,q)$  in these cases which will complete the proof. We first consider the case of $\calS_t^{=} \cap \calS_{t^*}^{>} = \emptyset$ and note that because $FPR_{t^*}(T^*,q^*)  =  FPR_t(T,q)$ by construction we must also have 

\[
\int_{s \in \calS_{t}^{>} \setminus \calS_{t^*}^{>} } p_0(s) + q \int_{s \in \calS_{t}^{=} } p_0(s)  = 
\int_{s \in \calS_{t^*}^{>} \setminus \calS_{t}^{>} } p_0(s) + q^* \int_{s \in \calS_{t^*}^{=} } p_0(s)
\]

because the summation over $\calS_{t}^{>} \cap \calS_{t^*}^{>}$ cancels out. It then further implies that $TPR_{t}(T,q) \leq TPR_{t^*}(T^*,q^*)$ if 

\[
\int_{s \in \calS_{t}^{>} \setminus \calS_{t^*}^{>} } p_1(s) + q \int_{s \in \calS_{t}^{=} } p_1(s)  \leq 
\int_{s \in \calS_{t^*}^{>} \setminus \calS_{t}^{>} } p_1(s) + q^* \int_{s \in \calS_{t^*}^{=} } p_1(s)
\]

In order to show $TPR_{t}(T,q) \leq TPR_{t^*}(T^*,q^*)$ it then suffices to show

\begin{multline*}
\left( \int_{s \in \calS_{t}^{>} \setminus \calS_{t^*}^{>} } p_1(s) + q \int_{s \in \calS_{t}^{=} } p_1(s) \right)  
\cdot\left( \int_{s \in \calS_{t^*}^{>} \setminus \calS_{t}^{>} } p_0(s) + q^* \int_{s \in \calS_{t^*}^{=} } p_0(s) \right) 
\leq
\\
\left( \int_{s \in \calS_{t^*}^{>} \setminus \calS_{t}^{>} } p_1(s) + q^* \int_{s \in \calS_{t^*}^{=} } p_1(s) \right) 
\cdot\left( \int_{s \in \calS_{t}^{>} \setminus \calS_{t^*}^{>} } p_0(s) + q \int_{s \in \calS_{t}^{=} } p_0(s) \right)
\end{multline*}

We are considering the case of $\calS_t^{=} \cap \calS_{t^*}^{>} = \emptyset$ which implies that for any $s \in ( \calS_{t^*}^{>} \setminus \calS_{t}^{>} ) \cup  \calS_{t^*}^{=} $ and $s' \in (\calS_{t}^{>} \setminus \calS_{t^*}^{>} ) \cup \calS_{t}^{=}$ we must have $\frac{p_1(s)}{p_1(s) + p_0(s)} \geq \frac{p_1(s')}{p_1(s') + p_0(s')}$ due to our definition of $t^*$. We can then expand both sides of the inequality and apply Lemma~\ref{lem:fpr_tpr_helper2} to like terms on each side to get our desired inequality and therefore $TPR_{t}(T,q) \leq TPR_{t^*}(T^*,q^*)$ giving a contradiction.

Finally we consider the case of $q \in \{0,1\}$. If $q = 0$ then we can use the same reasoning by simply getting rid of the $q \int_{s \in \calS_{t}^{=} } p_0(s) $ term. If $q=1$ then we will just add $\calS_t^{=}$ to the set $\calS_t^{+}$ and again use the same reasoning. Thus in both cases $TPR_{t}(T,q) \leq TPR_{t^*}(T^*,q^*)$ giving a contradiction and completing our proof.

\end{proof}

\subsection{Helper Lemmas for Lemma~\ref{lem:fpr_tpr_main}}

Our proof of Lemma~\ref{lem:fpr_tpr_main} will also require a few more general helper lemmas.

\begin{lemma}\label{lem:fpr_tpr_simple_helper}
For any $P_1, P_2 > 0$ and $q \in (0,1)$ there must exist $q_1 > q_2$ where $q_1 P_1 + q_2 P_2 = q(P_1 + P_2)$ such that either $q_1 = 1$ and $q_2 \geq 0$ or $q_2 = 0$ and $q_1 \leq 1$

\end{lemma}

\begin{proof}

If $q_1 = 1$ and $q_1 P_1 + q_2 P_2 = q(P_1 + P_2)$ then  $q_2 < 1$ because $P_1 + P_2 > q(P_1 + P_2)$ from the condition that $q \in (0,1)$. Further if $q_2 < 0$ when $q_1 = 1$ then this implies $P_1 > q(P_1 + P_2)$, and if $P_1 > q(P_1 + P_2)$ then there exists $q_1 \in (0,1)$ such that $q_1 P_1 = q(P_1 + P_2)$ because $P_1, P_2 > 0$.

\end{proof}

\begin{lemma}\label{lem:fpr_tpr_simple_helper2}
For any $P_1, P_2 > 0$ and $q \in (0,1)$ there must exist $q' \in (q,1)$ such that $q P_1 + P_2 = q'(P_1 + P_2)$

\end{lemma}

\begin{proof}

This follows immediately from the fact that $q P_1 + P_2 < P_1 + P_2$ and $q P_1 + P_2 > q(P_1 + P_2)$

\end{proof}

\begin{lemma}\label{lem:fpr_tpr_helper}

Given functions $p_o: \calX \rightarrow \R_{\geq 0}$ and $p_1: \calX \rightarrow \R_{\geq 0}$, let $\calX_{+} \subset \calX$ and $\calX_{-} \subset \calX$ such that $\calX_{+} \neq \emptyset$ and $\calX_{-} \neq \emptyset$ and they are disjoint so $\calX_{+} \cap \calX_{-} = \emptyset$. We further assume that for any $x \in \calX_{+}$ and $y \in \calX_{-}$ we have $p_1(x) + p_0(x) > 0$, $p_1(y) + p_0(y) > 0$  and 

\[
\frac{p_1(x)}{p_1(x) + p_0(x)} > \frac{p_1(y)}{p_1(y) + p_0(y)}
\]

For any $q_1,q_2 \in \R$ such that

\[
\int_{x\in \calX_{+}} p_0(x) + \int_{y \in \calX_{-}} p_0(x) =
q_1\int_{x\in \calX_{+}} p_0(x) + q_2\int_{y \in \calX_{-}} p_0(x)
\]

then if $q_1 > q_2$ we must have

\[
\int_{x\in \calX_{+}} p_1(x) + \int_{y \in \calX_{-}} p_1(x) <
q_1\int_{x\in \calX_{+}} p_1(x) + q_2\int_{y \in \calX_{-}} p_1(x)
\]

and the inequality is reversed if $q_2 > q_1$.

\end{lemma}

\begin{proof}
From our assumed inequality we have for any $x \in \calX_{+}$ and $y \in \calX_{-}$ that $p_1(x) p_0(y) > p_1(y) p_0(x)$ which implies

\[
\int_{x,y \in \calX_{+} \times \calX_{-}} p_1(x)p_0(y) > \int_{x,y \in \calX_{+} \times \calX_{-}} p_1(y)p_0(x)
\]

and using the assumption that $q_1 > q_2$ we have 

\[
(q_1 - q_2)\int_{x \in \calX_{+}} p_1(x) \int_{y \in \calX_{-}} p_0(y) >
(q_1 - q_2)\int_{y \in \calX_{-}}  p_1(y) \int_{x \in \calX_{+}}  p_0(x)
\]

Expanding this and rearranging gives

\begin{multline*}
q_1 \left( \int_{x \in \calX_{+}} p_1(x) \int_{y \in \calX_{-}} p_0(y) \right)  + 
q_2 \left( \int_{y \in \calX_{-}}  p_1(y) \int_{x \in \calX_{+}}  p_0(x) \right) > 
\\
q_1 \left( \int_{y \in \calX_{-}}  p_1(y) \int_{x \in \calX_{+}}  p_0(x) \right) + 
q_2 \left( \int_{x \in \calX_{+}} p_1(x) \int_{y \in \calX_{-}} p_0(y)\right)
\end{multline*}

Further expanding gives

\begin{multline*}
\left(   q_1\int_{x\in \calX_{+}} p_1(x) + q_2\int_{y \in \calX_{-}} p_1(x)  \right) 
\left(   \int_{x\in \calX_{+}} p_0(x) + \int_{y \in \calX_{-}} p_0(x)  \right)
> 
\\
\left(    \int_{x\in \calX_{+}} p_1(x) + \int_{y \in \calX_{-}} p_1(x)   \right)    
\left(   q_1\int_{x\in \calX_{+}} p_0(x) + q_2\int_{y \in \calX_{-}} p_0(x)  \right)
\end{multline*}

and applying our assumed equality implies the desired inequality. Note that if we instead assume $q_2 > q_1$ then the proof is identical but the inequality is flipped when each side is multiplied by $(q_1 - q_2)$ because it is negative.

\end{proof}

\begin{lemma}\label{lem:fpr_tpr_helper2}

Given functions $p_o: \calX \rightarrow \R_{\geq 0}$ and $p_1: \calX \rightarrow \R_{\geq 0}$, let $\calX_{+} \subset \calX$ and $\calX_{-} \subset \calX$ such that for any $x \in \calX_{+}$ and $y \in \calX_{-}$ we have $p_1(x) > 0$, $p_1(y) > 0$  and 
$
\frac{p_1(x)}{p_1(x) + p_0(x)} \geq \frac{p_1(y)}{p_1(y) + p_0(y)}
$ then

\[
\int_{x\in \calX_{+}} p_1(x) \int_{y\in \calX_{-}} p_0(y) \geq  \int_{x\in \calX_{+}} p_0(x) \int_{y\in \calX_{-}} p_1(y)
\]

\end{lemma}

\begin{proof}
This follows immediately from our assumed inequality that can be equivalently written as $p_1(x) p_0(y) \geq p_1(y)p_0(x)$ for any $x \in \calX_{+}$ and $y \in \calX_{-}$

\end{proof}

\subsection{Calibrated Precision-Recall}

Another common metric for accuracy is consideration of Precision with respect to Recall, which also has an associated AUC metric.
We additionally give the calibrated versions of these metrics and show that the optimal transformation for AUC will also optimize these metrics.

We first note that Lemma~\ref{lem:fpr_tpr_main} could have equivalently flipped FPR and TPR.

\begin{corollary}\label{cor:fpr_tpr_main}

For any $x \in [0,1]$ we have $FPR_{t^*}(TPR_{t^*}^{-1}(x)) \leq FPR_t(TPR_t^{-1}(x))$ where $t^*$ is defined the same as in Lemma~\ref{lem:fpr_tpr_main}

\end{corollary}

\begin{proof}
The proof of this follows symmetrically to that of Lemma~\ref{lem:fpr_tpr_main} where the thresholds are now defined in terms of TPR and everything is flipped so the inequalities still hold in the correct direction.
\end{proof}

With this corollary in hand we will now define calibrated precision and recall.

\begin{definition}[Calibrated Recall]
Identical to Definition~\ref{def:calibrated_tpr}
\end{definition}

\begin{definition}[Calibrated Precision]
For a given classifier score function $\s: \calX \rightarrow \R$, and distribution $\calD_1$, along with a transformation function $t: \R \rightarrow \R$, and some $T \in \R$ and $q \in [0,1]$, we define calibrated Precision as

\[
\text{Precision}_t(T,q) = \frac{TPR_t(T,q)}{TPR_t(T,q) + FPR_t(T,q)}
\]
\end{definition}

In the same way as TPR and FPR, it is then common to compute Precision with the thresholds of Recall for some $x\in [0,1]$, which is to say that we consider $\text{Precision}_t(\text{Recall}_t^{-1}(x))$.
This yields an similar area under the curve metric for which we give the calibrated version.

\begin{definition}\label{def:pr-auc}[Calibrated PR-AUC]

For a given classifier score function $\s: \calX \rightarrow \R$, and distribution $\calD_1$, along with a transformation function $t: \R \times \Pi \rightarrow \R$, and some $T \in \R$ and $q \in [0,1]$, we define calibrated Precison-Recall AUC as

\[
\text{PR-AUC}(\s, \calD_0,\calD_1,t) = \int_0^1 \frac{x}{x + FPR_t(TPR_t^{-1}(x))} dx
\]

\end{definition}

Once again, we show that this metric is optimized by applying the optimal transformation for AUC.

\begin{lemma}

For a given classifier score function $\s: \calX \rightarrow \R$, and distributions $\calD_0$ and $\calD_1$, let $t^*:\R  \rightarrow [0,1]$ be the transformation 

\[
t^*(s) \defeq 
\frac{\prob{\xx \sim \calD_1}{\s(\xx) = s}}{\prob{\xx \sim \calD_1}{\s(\xx) = s} + \prob{\xx \sim \calD_0}{\s(\xx) = s}} 
\] 

and if ${\prob{\xx \sim \calD_1}{\s(\xx) = s} + \prob{\xx \sim \calD_0}{\s(\xx) = s}} = 0$ then let $t^*(s) = 0$.

For any function $t: \R \rightarrow \R$ we have $\text{PR-AUC}(\s, \calD_0,\calD_1,t^*) \geq \text{PR-AUC}(\s, \calD_0,\calD_1,t)$

\end{lemma}

\begin{proof}
Follows immediately from Corollary~\ref{cor:fpr_tpr_main} and the fact that by definition FPR must always be non-negative.

\end{proof}

\subsection{Extension to Partition Calibrated}

All of these proofs can equivalently be extended to partition calibrated where we avoided this more general case for notational simplicity as the proofs are already quite involved.
We first note that the underlying intuition here is still equivalent in that the highest ordered outputs from the transformation will still maximize TPR while minimizing FPR.
More explicitly, if we consider the proof of Lemma~\ref{lem:fpr_tpr_main}, for the given threshold we defined the sets under which the output scores had a transformed output above or equal to that threshold.
Under the partititioning these would instead be subsets of $(\R \times \Pi)$ but could otherwise be defined in the same manner.
Furthermore all the inequality properties of these respective sets that were used to give the proof would still hold due to construction of the optimal transformation.
In particular, borrowing the same notation of the proof of Lemma~\ref{lem:fpr_tpr_main} and slightly generalizing to $p_1(s,\Pi_i)$, we would still have the same inequality properties of these $p_1(s,\Pi_i)$ for when $(s,\Pi_i)$ was from the optimal set vs the other transformation set as a result of the construction.
Additionally the helper lemmas were already generalized enough to be applied to this setting.
The rest of the proofs would then equivalently follow as they were all a straightforward application of  Lemma~\ref{lem:fpr_tpr_main}

\section{Partition Calibrated Log-loss Proofs}\label{sec:log_loss_proofs} 

In this section we will prove that the optimal transformation for calibration, and thus also AUC and other metrics as has been shown in Section~\ref{sec:auc_proofs} and~\ref{sec:fpr_tpr}, will also optimize the log-loss.
Accordingly, our heterogeneous calibration will theoretically optimize over a variety of metrics for model performance.
Additionally we will define expected calibration error later in this section for completeness.
We begin by defining log-loss and recall that we let $d:\calX \times \calY \rightarrow [0,1]$ be the density function of distribution $\calD$.

\begin{definition}\label{def:log_loss}[Log-loss]
For a given classifier score function $s: \calX \rightarrow \R$, along with distribution $ \calD$ the log loss is

\[
L(\s,\calD) = -\int_{(\xx,y) \in ({\calX},\calY)} d(\xx,y) \bigg(y \cdot \log (\sigma(\s(\xx))) + (1-y) \log (1-\sigma(\s(\xx))) \bigg)
\]

where $\sigma: \R \rightarrow [0,1]$ is the sigmoid function such that $\sigma(x) = \frac{1}{1 + e^{-x}}$

\end{definition}

In addition we will define the partition calibrated log-loss.

\begin{definition}\label{def:cal_log_loss}[Partition Calibrated Log-loss]
Given distribution $\calD$ and a partition $\Pi$ of the feature space, along with score function $s: \calX \rightarrow \R$, we let $\calS_0 \subseteq \R$ be such that $s \in \calS_0$ if and only if $\prob{(\xx,y) \sim \calD}{y=0,\s(\xx) = s} > 0$ and similarly let $\calS_1 \subseteq \R$ be such that $s \in \calS_1$ if and only if $\prob{(\xx,y) \sim \calD}{y=0,\s(\xx) = s} > 0$. For a given transformation function $t: \R \times \Pi \rightarrow [0,1]$ we define the generalized log-loss to be

\begin{align*}
&L(\s,\calD, t, \Pi) = 
\\& \int_{\Pi_i \in \Pi}
\Bigg( -\int_{s \in \calS_1} \prob{(\xx,y) \sim \calD}{y=1,\s(\xx) = s, \xx \in \Pi_i} \log (t(s)) 
- \int_{s \in \calS_0}\prob{(\xx,y) \sim \calD}{y=0,\s(\xx) = s, \xx \in \Pi_i} \log (1-t(s)) \Bigg)
\end{align*}

\end{definition}

Note that if $t$ is the sigmoid function regardless of the partition then this is equivalent to the log-loss.
It will then be straightforward to show that the optimal calibration of each partition will minimize the partition calibrated log-loss.

\begin{corollary}
Given distribution $\calD$ and our classifier score function $\s: \calX \rightarrow \R$. If we define $t^*:\calS_0 \cup \calS_1 \times \Pi \rightarrow [0,1]$ to be the transformation $t^*(s, \Pi_i) = \prob{(\xx,y) \sim \calD}{y=1 |  \s(\xx) = s, \xx \in \Pi_i}$, then for any $t:\R \times \Pi \rightarrow [0,1]$ we have $L(\s, \calD,t^*, \Pi) \leq L(\s,\calD,t,\Pi)$

\end{corollary}

\begin{proof}
We can simply consider each $\Pi_i$ separately and show that it is optimized by $t^*$ and so we will drop the $\Pi_i$ notation for simplicity.

For any $s \in \calS_0 \cup \calS_1$ we first note that by the definition of conditional probability we have 

\[
\prob{(\xx,y) \sim \calD}{y=1 |  \s(\xx) = s} 
= \frac{\prob{(\xx,y) \sim \calD}{y=1,\s(\xx) = s}}{\prob{(\xx,y) \sim \calD}{\s(\xx) = s}} 
= \frac{\prob{(\xx,y) \sim \calD}{y=1,\s(\xx) = s}}{\prob{(\xx,y) \sim \calD}{y=1,\s(\xx) = s} + \prob{(\xx,y) \sim \calD}{y=0,\s(\xx) = s}}
\]

which we will use alternatively as the transformation and we have three cases. If $s \in \calS_0 \setminus \calS_1$ then $t^*(s) = 0$ and this minimizes $- \prob{(\xx,y) \sim \calD}{y=0,\s(\xx) = s} \log (1-t(s))$. If $s \in \calS_1 \setminus \calS_0$ then $t^*(s) = 1$ and this minimizes $- \prob{(\xx,y) \sim \calD}{y=1,\s(\xx) = s} \log (t(s))$.
Finally if $s \in \calS_0 \cap \calS_1$ then we have 

\[
\frac{dL}{dt(s)} = - \frac{ \prob{(\xx,y) \sim \calD}{y=1,\s(\xx) = s} }{t(s)} 
+ \frac{ \prob{(\xx,y) \sim \calD}{y=0,\s(\xx) = s} }{1-t(s)}
\]

Setting $\frac{dL}{dt(s)} = 0$ and solving for $t(s)$ gives our $t^*(s)$ and therefore $t^*$ minimizes our loss function.

\end{proof}

\subsection{Calibration error}\label{subsec:cal_def}

We will also define calibration error in the same way as~\cite{kumar2019verified} and is mostly standard but has been translated to our notation.
Recall that we let $d:\calX \times \calY \rightarrow [0,1]$ be the density function of distribution $\calD$.

\begin{definition}\label{def:calibration_error}[Calibration Error]
For a given classifier probability function $p: \calX \rightarrow [0,1]$, along with distribution $ \calD$, the calibration error is
\[
CE(p,\calD) = \bigg( \int_{(\xx,y) \in (\calX,\calY)} d (\xx,y)\big( p(\xx) - \prob{(\xx,y) \sim \calD}{y=1 | \xx} \big)^2 \bigg)^{\frac{1}{2}}
\]
\end{definition}

This metric cannot actually be measured empirically because $\prob{(\xx,y) \sim \calD}{y=1|\xx}$ is exactly what our model is trying to find.   Instead we need to group together all the feature vectors with the same score from our classifier function because we can get an accurate estimate on $\prob{(\xx,y) \sim \calD}{y=1|\s(\xx) = s}$ for a given score $s$ by binning the possible scores.
This gives an empirical estimate of the following definition.

\begin{definition}\label{def:expected_calibration_error}[Expected Calibration Error]
For a given classifier score function $\s: \calX \rightarrow \R$, along with distribution $ \calD$ and some calibration transformation function $t: \R \rightarrow [0,1]$, the expected calibration error of this transformation function is

\[
CE(\s,\calD,t) = \bigg( \int_{s \in \R} \prob{(\xx,y) \sim \calD}{\s(\xx) = s} \big( t(s) - \prob{(\xx,y) \sim \calD}{y=1 | \s(\xx) = s} \big)^2 \bigg)^{\frac{1}{2}}
\]

When $t$ is the sigmoid function $\sigma$ then this is the expected calibration error of our classifier function without any calibration post-processing.

\end{definition}

Note that low expected calibration error is not necessarily correlated with the performance of the classifier function. For instance if we suppose our labels are balanced, so $\prob{(\xx,y) \sim \calD}{y=1} = \prob{(\xx,y)\sim \calD}{y=0}$, and our score function is simply $\s(\xx) = 0$ for all $\xx$, then by setting $t$ as the sigmoid function $\sigma$ we get $\sigma(\s(\xx)) = \frac{1}{2}$ for all $\xx$. Accordingly we can actually see that our expected calibration error is zero and has been perfectly minimized but our classifier function is performing as poorly as possible.

As a result, by partitioning the feature space and calibrating each separately, we would not necessarily expect to improve the expected calibration error because unlike with AUC or log-loss we cannot further improve the expected calibration error by optimally calibrating each partition.
Note that optimally calibrating each partition will still theoretically improve the calibration error.




\section{Splitting criteria for optimizing AUC}\label{sec:algorithm} 

In this work we primarily focus on heterogeneity that is based upon label balance, however it could be that the heterogeneity is because different partitions are easier to classify or have higher variance.
Identifying the optimal partitioning that could improve AUC would then require a more complex algorithm than a simple decision tree and would also need to utilize the scores from the classifier function in finding the partitioning.
If the goal is to partition the feature space then using a decision tree style algorithm would be the most obvious approach.
We would then want the splitting criteria function to be more specified to the amount we can increase AUC by making that given split.
However we still want to be able to execute this tree at scale so iterating through the splits and thus computing the splitting criteria function should be done efficiently.

Even if the criteria function was just the weighted average of the AUC on each side of the split, it would still be computationally expensive as AUC computation requires sorting the data and this sorted list cannot be efficiently updated as we iterate through the possible splits. 
Standard criteria functions for CART will only require aggregate statistics that can be efficiently updated and can be used to compute simple statistics such as mean and variance.
Recall that AUC is computed from the score distributions for positive and negative labels, so for each side of the split we could implicitly track mean and variance for both the positive labels and negative labels and efficiently update these as we iterate through the possible splits.
Due to the fact that we cannot efficiently track the empirical distributions for positive and negative scores we will then approximate these distributions with normal distribution with our empirical mean and variance.

Assuming that the positive and negative score distributions follow this normal distribution would allow for efficient computation of AUC using the CDF of each distribution.
However we are specifically trying to identify partitions where the relative ordering should be changed so we do not want our splitting criteria to be the weighted average of the AUC on each side of the split.
Instead the splitting criteria would ideally be the AUC after applying calibration to each side of the split.

In order to more explicitly specify how we can apply calibration on each side of the split, we consider the split to be some partitioning $\Pi = \{\Pi_l, \Pi_r\}$. We can efficiently track the mean and variance of the respective score distributions, and we will assume they follow the normal distribution, which gives $\prob{\xx \sim \calD_0}{\s(\xx) = s| \xx \in \Pi_l} = \calN(s;\mu_{0,l},\sigma_{0,l}) $ and similarly for the other respective label and partition pairs. Each of these means and variances are estimates can be efficiently updated in the same manner as is done in CART algorithms where we track the sum of the scores, the sum of the scores squared, and the total points which can then be used to quickly compute these values, and additionally this is separately done for when label is positive or negative. Furthermore, we can efficiently the number of positive and negative labels in each side of the split which gives an estimate on $p_{l} = \prob{(\xx,y)\in \calD}{y=1 | \xx \in \Pi_l}$ and similarly for the other respective label and partition pairs. We can then use this to get an estimate of

\[
\prob{(\xx,y) \in \calD}{y=1|\s(\xx) = s, \xx \in \Pi_l} = 
\frac{p_l \calN(s;\mu_{1,l},\sigma_{1,l})}{p_l \calN(s;\mu_{1,l},\sigma_{1,l}) + (1 - p_l) \calN(s;\mu_{0,l},\sigma_{0,l})}
\]

The goal would then be to apply a calibration technique to each side of the split, then compute the AUC of the full dataset with the transformation applied. In order to keep this computationally feasible, we will apply Platt scaling which is just logistic regression, and thus would just apply a linear transformation to each side of the split. For a normal distribution we have that $\prob{\xx \sim \calD_0}{a_l\s(\xx) + b_l = s| \xx \in \Pi_l} = \calN(s;a_l\mu_{0,l} + b_l, a_l \sigma_{0,l}) $. Accordingly, the logistic regression would try perfectly calibrate the partition which is equivalent to solving for $a_l,b_l$ such that for any $s \in \R$ we have

\[
s = \sigma^{-1} \left(
\prob{(\xx,y) \in \calD}{y=1|a_l\s(\xx) + b_l = s, \xx \in \Pi_l} \right)
\]

This can then be reduced to

\[
s = \log \left(  \frac{p_l \calN(s;a_l\mu_{1,l} + b_l,a_l \sigma_{1,l})}{(1 - p_l) \calN(s;a_l\mu_{0,l} + b_l,a_l \sigma_{0,l})} \right) 
\]

Plugging in the normal distribution and applying cancellation gives

\[
    \frac{p_l \calN(s;a_l\mu_{1,l} + b_l,a_l \sigma_{1,l})}{(1 - p_l) \calN(s;a_l\mu_{0,l} + b_l,a_l \sigma_{0,l})} = 
    \frac{p_l \sigma_{0,l}}{(1-p_l)\sigma_{1,l}} e^{-\frac{1}{2}\left( \frac{s - a_l \mu_{1,l} - b_l}{a_l\sigma_{1,l}} \right)^2 + \frac{1}{2}\left( \frac{s - a_l \mu_{0,l} - b_l}{a_l\sigma_{0,l}} \right)^2}
\]

Therefore, we want to solve for $a_l,b_l$ such that for all $s \in \R$

\[
s + \log \left( \frac{(1-p_l)\sigma_{1,l}}{p_l \sigma_{0,l}}  \right) =
-\frac{1}{2}\left( \frac{s - a_l \mu_{1,l} - b_l}{a_l\sigma_{1,l}} \right)^2 + \frac{1}{2}\left( \frac{s - a_l \mu_{0,l} - b_l}{a_l\sigma_{0,l}} \right)^2
\]

There may not be $a_l,b_l$ such that this holds for all $s \in \R$ and the logistic regression would find the closest possible fit. We will approximate this by taking two values $s_1,s_2$ close to $\log(\frac{p_l}{1-p_l})$, which is where the distributions should have the same probability, and solve for $a_l,b_l$ with the associated quadratic equations.

Accordingly we would now have the calibrated distributions on each side of the split where $t_l(s) = a_l s + b_l$ and $t_r(s) = a_r s + b_r$, which would give $\prob{\xx \sim \calD_0}{t_l(\s(\xx)) = s| \xx \in \Pi_l} = \calN(s;a_l\mu_{0,l} + b_l, a_l \sigma_{0,l})$ and $\prob{\xx \sim \calD_0}{t_r(\s(\xx)) = s| \xx \in \Pi_l} = \calN(s;a_r\mu_{0,r} + b_r, a_r \sigma_{0,r})$, and identically for positive labels.

For a given threshold $T$ we could then use the CDF of the normal distribution to compute the TPR of the full partition as 

\[
\int_{T}^{\infty} \bigg( \frac{p_l}{p_l + p_r} \calN(s;a_l\mu_{1,l} + b_l, a_l \sigma_{1,l}) 
+ \frac{p_r}{p_l + p_r} \calN(s;a_r\mu_{1,r} + b_r, a_r \sigma_{1,r}) \bigg) ds
\]

This would equivalently be done for FPR and would allow us to more efficiently compute the AUC for the calibrated partitions. 
As such we would have the approximate splitting criteria for maximizing AUC that could be efficiently computed.

\section{Applying Gradient Boosted Trees to Heterogeneous Calibration Framework}\label{sec:boosted} 

In this section we discuss applying boosted trees to the heterogeneous calibration framework.
We leave the full details to future work, but sketch out how this framework can utilize gradient boosted trees for completeness.

For simplicity we will restrict the number of estimators to 2 for the gradient boosted tree. 
Accordingly, it will give a sequence of two partitions of the feature space where we let $\Pi^1, \Pi^2$ be the partitioning of the first and second estimator respectively.
We also assume that the calibration is Platt scaling (logistic regression) for simplicity.
In the same way that Algorithm~\ref{algo:hetcal} doesn't use the probability predictions of the decision but only the generated partitioning, we will only use the partitions $\Pi^1, \Pi^2$.
The high-level idea is that calibration will be applied to all the partitions of $\Pi^1$ separately and then we calibrate these new scores separately for all the partitions of $\Pi^2$.
The intuition for applying this sequential calibration still follows similarly where boosted trees capture more granular heterogeneity but do so in a way that avoids building high depth trees and thus losing empirical accuracy.
This avoidance will also help maintain accuracy for the successive calibration calls instead of giving a more fine-grained partitioning.

Let's denote our heterogeneous calibration as the transformation $t$ such that for any feature vector $\xx \in \Pi_i$ we output  $t(\s(\xx),\Pi_i)$, where $\s$ is the classifier score function of the respective DNN being calibrated.
This sequential calibration will then give the full transformation of $t_2(t_1(\s (\xx), \Pi^1_i),\Pi^2_j)$ where $\xx \in \Pi^1_i$ and $\xx \in \Pi^2_j$.
If we assume that the calibration is Platt scaling then it will take the form of $\sigma(a_{2,j}(a_{1,i} \s (\xx) + b_{1,i}) + b_{2,j}) $, where we note that Platt scaling should be applied to scores that have not undergone the sigmoid transformation.
Thus, the second layer of Platt scaling is applied to the scores $a_{1,i} \s (\xx) + b_{1,i}$ for each $\xx \in \Pi^1_i$.
This can then easily be extended to a longer sequence of partitions.


We also discussed in Section~\ref{subsec:interpolation} how the calibration could be seen as an interpolation between DNNs and boosted trees. It is clear from our example that if the logistic regression always gives the identity as it's respective linear function then this will simply output our original DNN. 
The fact that it can also capture the original gradient boosted tree is slightly more difficult.
The gradient boosted tree model has respective constants $c_{1,i}$ and $c_{2,j}$ for each partition $\Pi^1_i \in \Pi^1$ and $\Pi^2_j \in \Pi^2$, and gives probability output estimate $c_{1,i} + c_{2,j}$ for $\xx \in \Pi_i^1 \cap \Pi_j^2$.
The subsequent constants represent the residual error of the composition of the previous partitions.
Our calibration then needs to be able to mimic this addition of constants, which is to say that we want $t_1(s, \Pi^1_i) = c_{1,i}$ and $t_2(s,\Pi^2_j) = s + c_{2,j}$ for any score $s \in \R$.
Due to the non-linearity of the sigmoid function on the final layer, this cannot necessarily be done through Platt scaling.
However we note that Platt scaling also doesn't generally allow for perfect calibration and is simply a good proxy and especially useful when there are few samples due to it's simplicity.
Furthermore, this type of tranformation can easily be captured through another common calibration technique, isotonic regression.

As such, calibration can also be seen as an interpolation between DNNs and gradient boosted trees.
The optimality analysis should also extend similarly when we instead consider this sequence of partition, but we leave the full details to future work.

\section{More Experimental results}\label{sec:experiments_app} 

\begin{table*}[!ht]
\begin{tabular}{lllllll}
\hline
Size                        & Model                        & Bank Marketing  & Census data & Credit Default & Higgs data  & Diabetes \\ \hline
S  & Reg DNN               & 0.7758           & 0.8976           & 0.7781   & 0.7801     & 0.6693  \\
                            & Reg HC           &  0.7816 (+0.76\%)              &   0.9021 (+0.50\%)     &  0.7793 (+0.16\%)      &  0.7816 (+0.19\%)        &  0.6829 (+2.04\%)    \\ \cline{2-7} 
                            & Unreg DNN     &   0.7735               &        0.8773        &   0.7768     &    0.7498      &  0.6915     \\
                            & Unreg HC &   0.7804 (+0.88\%)             &   0.8985 (+2.42\%)     &  0.7787 (+0.25\%)     &   0.7588 (+1.20\%)       &  0.6937 (+0.32\%)     \\ \hline \hline
M & Reg DNN               & 0.7712           & 0.8978            & 0.7787    & 0.7773      & 0.6951   \\
                            & Reg HC           &  0.7800 (+1.14\%)     &          0.9027 (+0.55\%)      &  0.7794 (+0.09\%)      &  0.7799 (+0.33\%)        &   0.6832 (+3.65\%)    \\ \cline{2-7} 
                            & Unreg DNN    &   0.7676     &           0.8644     & 0.7768      &  0.7477        &   0.6744   \\
                            & Unreg HC &  0.7770 (+1.22\%)   &              0.8994 (+4.05\%)  &  0.7786 (+0.23\%)      &    0.7586 (+1.45\%)      &   0.6856 (+1.66\%)    \\ \hline \hline
L  & Reg DNN              & 0.7707           & 0.9007            & 0.7782    & 0.7747      & 0.6567   \\
                            & Reg HC           &    0.7809 (+1.32\%)            &    0.9027 (+0.22\%)            &  0.7795 (+0.17\%)      &   0.7775 (+0.36\%)       &  0.6824 (+3.91\%)     \\ \cline{2-7} 
                            & Unreg DNN     &  0.7642              &            0.8437    &  0.7771     &        0.7487  &  0.6679    \\
                            & Unreg HC & 0.7770 (+1.67\%)               &      0.8980 (+6.44\%)          & 0.7787n(+0.21\%)       & 0.7595 (+1.44\%)         &  0.6824 (+2.17\%)     \\ \hline \hline
\end{tabular}
\caption{Test effect of regularization on AUC-ROC (mean of 5 runs) on different datasets before and after calibration. DNN = Deep neural network, HC = Heterogeneous calibration, Reg = regularized model, Unreg = unregularized model. We report model performance on the top 3 variants without regularization as well as the top 3 variants with regularization for each model, where top 3 is determined by DNN performance prior to HC}
\label{tab:reg_results_appendix}
\end{table*}

\subsection{Datasets and Pre-processing details}
\label{app:datasets}
Table~\ref{tab:datasetstats} contains the training/validation/calibration/test splits of the 5 datasets used in our experiments. Note that we use only $50,000$ datapoints from the original Higgs dataset.

Dataset pre-processing details can be found below:
\begin{itemize}
    \item \textit{Bank Marketing} - We remove the \textbf{duration} and \textbf{day} features. All categorical features are one-hot encoded, except for \textbf{month} and \textbf{job} for which we use dense embeddings. 
    \item \textit{Census Income} - The target label is whether income exceeds \$50k or not. All categorical features are one-hot encoded, except for \textbf{education}, \textbf{occupation} and \textbf{native country} for which we learn embeddings. 
    \item \textit{Credit Default} - The target label is whether a customer will default on payment next month or not. All features are either numerical or categorical (one-hot encoded).
    \item \textit{Higgs} - The target label is whether process is a background process or a Higgs boson producing signal process. All features are numerical.
    \item \textit{Diabetes} - The target label is whether a patient is readmitted to the hospital or not. We drop sparse features like \textbf{encounter id}, , \textbf{patient nbr}, \textbf{payer code}, \textbf{weight}, \textbf{max glu serum} and \textbf{A1Cresult}. For categorical features, we use embeddings for \textbf{medical specialty}, \textbf{discharge disposition id}, \textbf{admission source id}, \textbf{diag 1}, \textbf{diag 2}, \textbf{diag 3} and one-hot encoding for the rest.
\end{itemize}


\begin{table*}[h!]

  \begin{center}

    \begin{tabular}{l|c|c|c|r}
      \toprule 
      \textbf{Dataset} & \textbf{n(train)} & \textbf{n(validation)} & \textbf{n(calibration)} & \textbf{n(test)}\\
      \midrule 
      Bank Marketing & 29387 & 5322 & 5222 & 5222\\
      Census Income & 26048 & 6513 & 8141 & 8140\\
      Credit Default & 16500 & 4590 & 4455 & 4455\\
      Higgs & 35000 & 5100 & 4950 & 4950\\
      Diabetes & 76324 & 8650 & 8396 & 8396\\
      \bottomrule 
    \end{tabular}
  \end{center}
  \caption{Datasets splits for experiments.}
  \label{tab:datasetstats}


\end{table*}


\subsection{Model and hyperparameter details}
\label{app:modelhp}

\textbf{Featurization} For categorical features, we either use one-hot vectors or use embedding tables with a fixed dimension of 8. The input to the models then becomes a vector of fixed length.

\textbf{Model type and size} We use multilayer perceptrons with 3 layers. We use 3 model sizes, viz. small, medium and large. The number of neurons in each layer is $64 * i$, $32 * i$ and $16 * i$, where $i \in \{1, 2 , 4\}$ and each value of $i$ maps to the 3 aforementioned model sizes. Each layer uses the ReLU nonlinearity~\cite{nair2010rectified}. 

\textbf{Hyperparameters and regularization} We use a fixed batch size of 128 and train all models for 100 epochs each. For each model variant, we sweep learning rate over the set $\{5e-6, 1e-5, 5e-5, 1e-4, 5e-4, 1e-3, 5e-3, 1e-2\}$. We perform model selection by choosing the model with the best validation AUC over 100 epochs.

For each model size, we consider 3 regularization variants:

\begin{itemize}
    \item No regularization
    \item Batch normalization after each feedforward layer
    \item Dropout (with a fixed rate of 0.25) after each feedforward layer
\end{itemize}

For the credit default dataset, we additionally use a batch normalization layer at the input layer to improve the stability of training. 

Combining model size and regularization yields a total of 9 variants for each dataset. We conduct 5 runs per variant and report the mean of the runs.

\subsection{Decision tree and Platt scaling details}
\label{app:dt_and_platt_scaling}

For training a decision tree over the training data, we try 2 values each for the 2 hyperparameters - maximum tree depth (chosen from the set $\{3, 4\}$) and the minimum number of samples required for a node to be at a leaf node (chosen from the set $\{1000, 2000\}$). We note consistent lifts in performance despite devoting minimal effort to tuning the tree algorithm.

Platt scaling is performed on each partition of the calibration data by training a logistic regression on each partition's neural network logit scores.

\subsection{More results}
\label{app:more_results}

We further analyze the variation in lift seen from our heterogeneous calibration and provide box plots for all 5 runs of each of the top 3 models for a given size, where the metric considered is the percent lift in AUC.
As can be seen in these plots, even for the datasets where the lift is comparatively smaller we still see that it consistently improves the AUC across runs.

Figures~\ref{fig:box_plot_bm}-~\ref{fig:box_plot_higgs} contain box plots for AUC lift for heteregeneous calibration across the top 3 models for all 3 model sizes.
\begin{figure}[h!]
\centering
\includegraphics[width=0.45\textwidth]{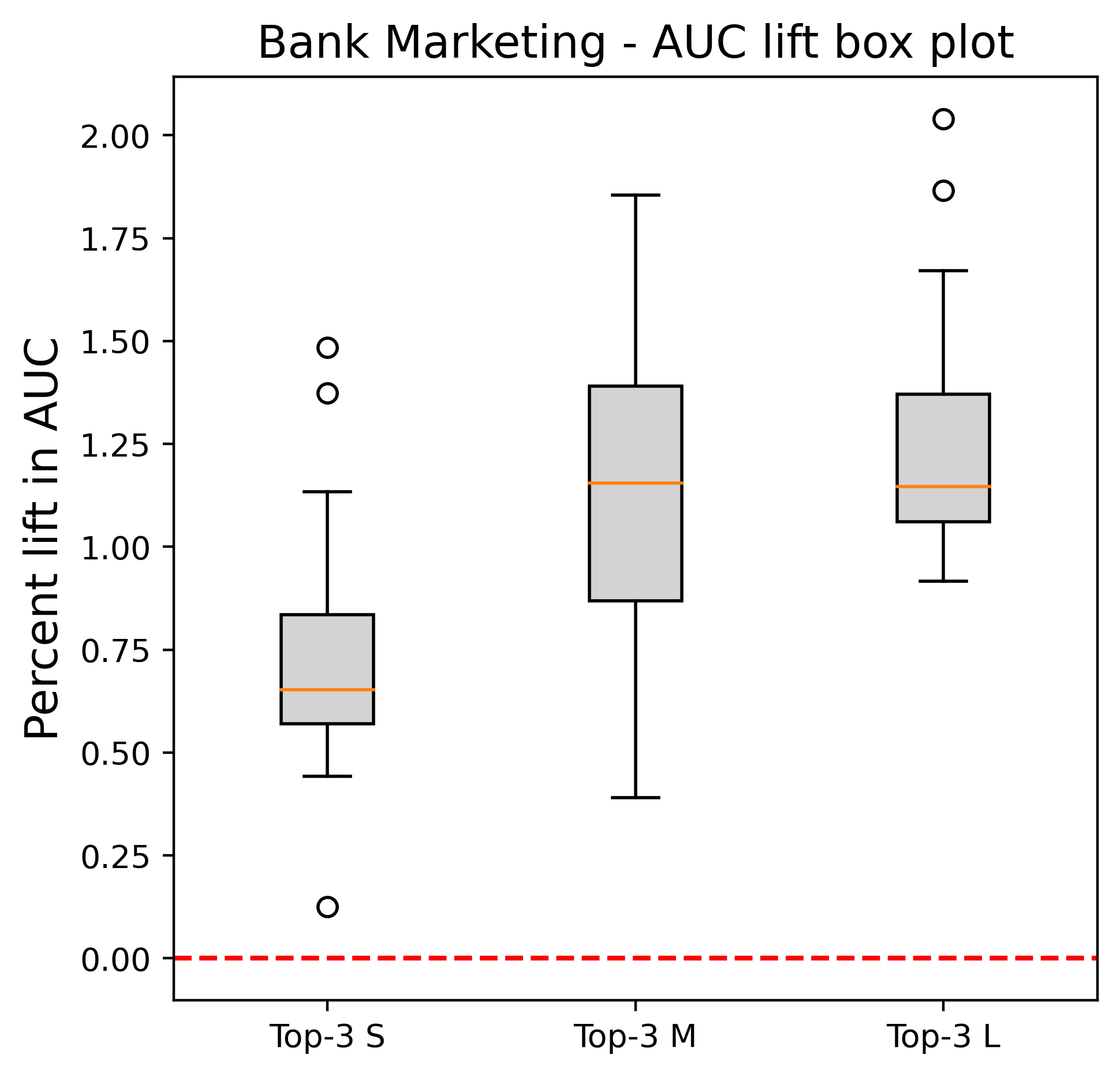}
\caption{Box plot for Bank Marketing - Test AUC lifts of our method for various runs of the top 3 models across model sizes.}
\label{fig:box_plot_bm}
\end{figure}

\begin{figure}[h!]
\centering
\includegraphics[width=0.45\textwidth]{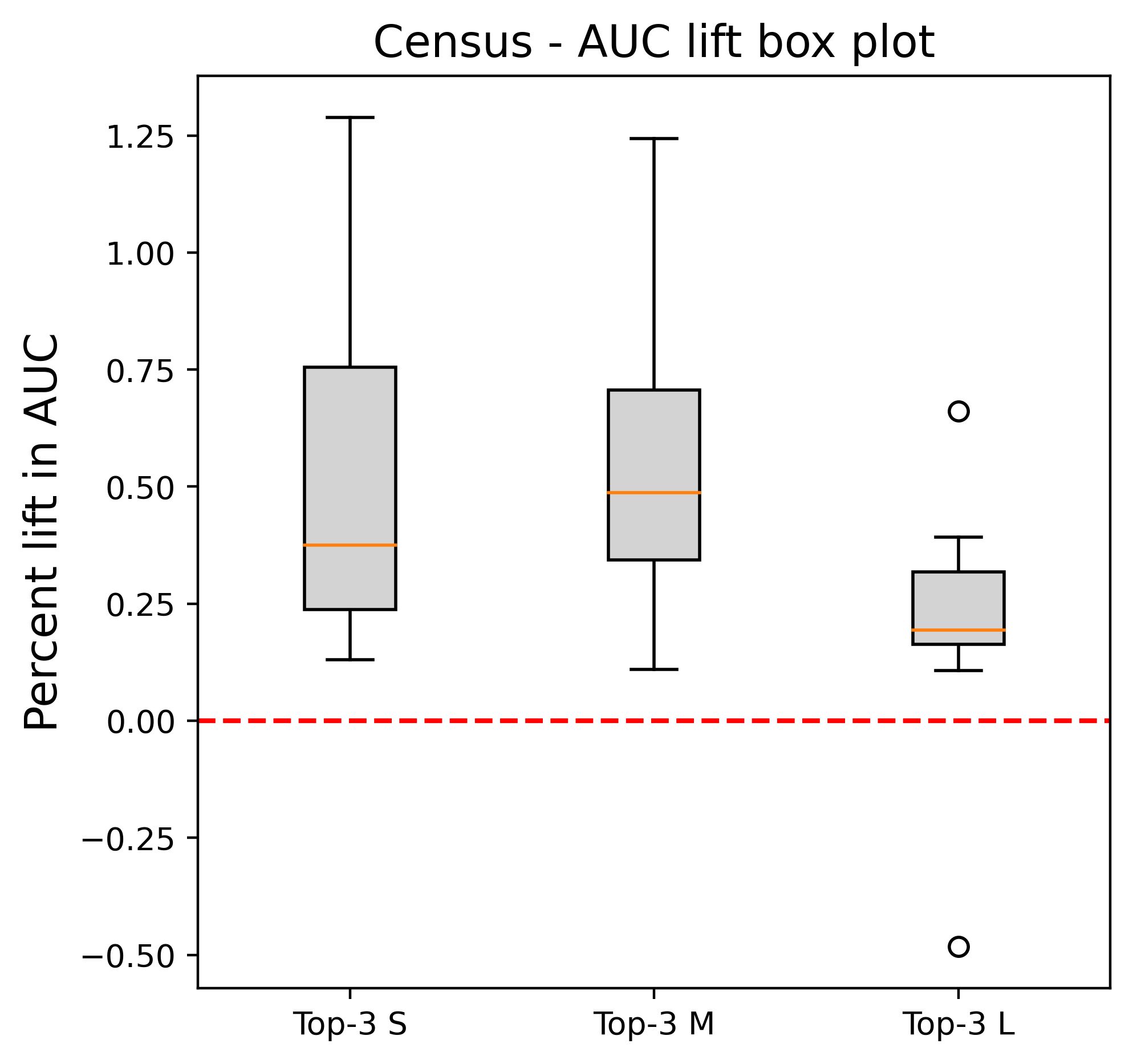}
\caption{Box plot for Census - Test AUC lifts of our method for various runs of the top 3 models across model sizes.}
\label{fig:box_plot_census}
\end{figure}

\begin{figure}[h!]
\centering
\includegraphics[width=0.45\textwidth]{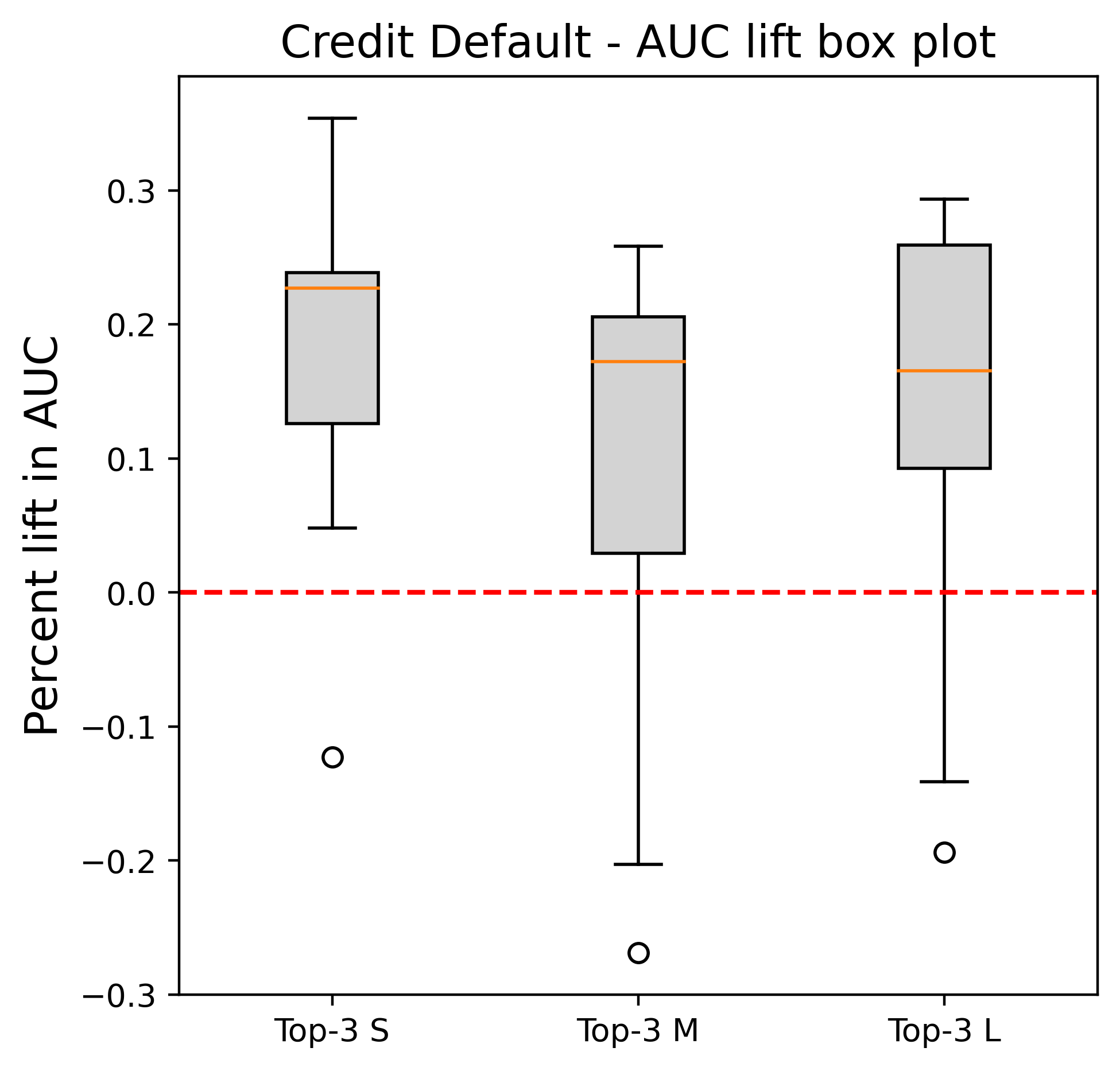}
\caption{Box plot for Credit Default - Test AUC lifts of our method for various runs of the top 3 models across model sizes.}
\label{fig:box_plot_credit}
\end{figure}

\begin{figure}[h!]
\centering
\includegraphics[width=0.45\textwidth]{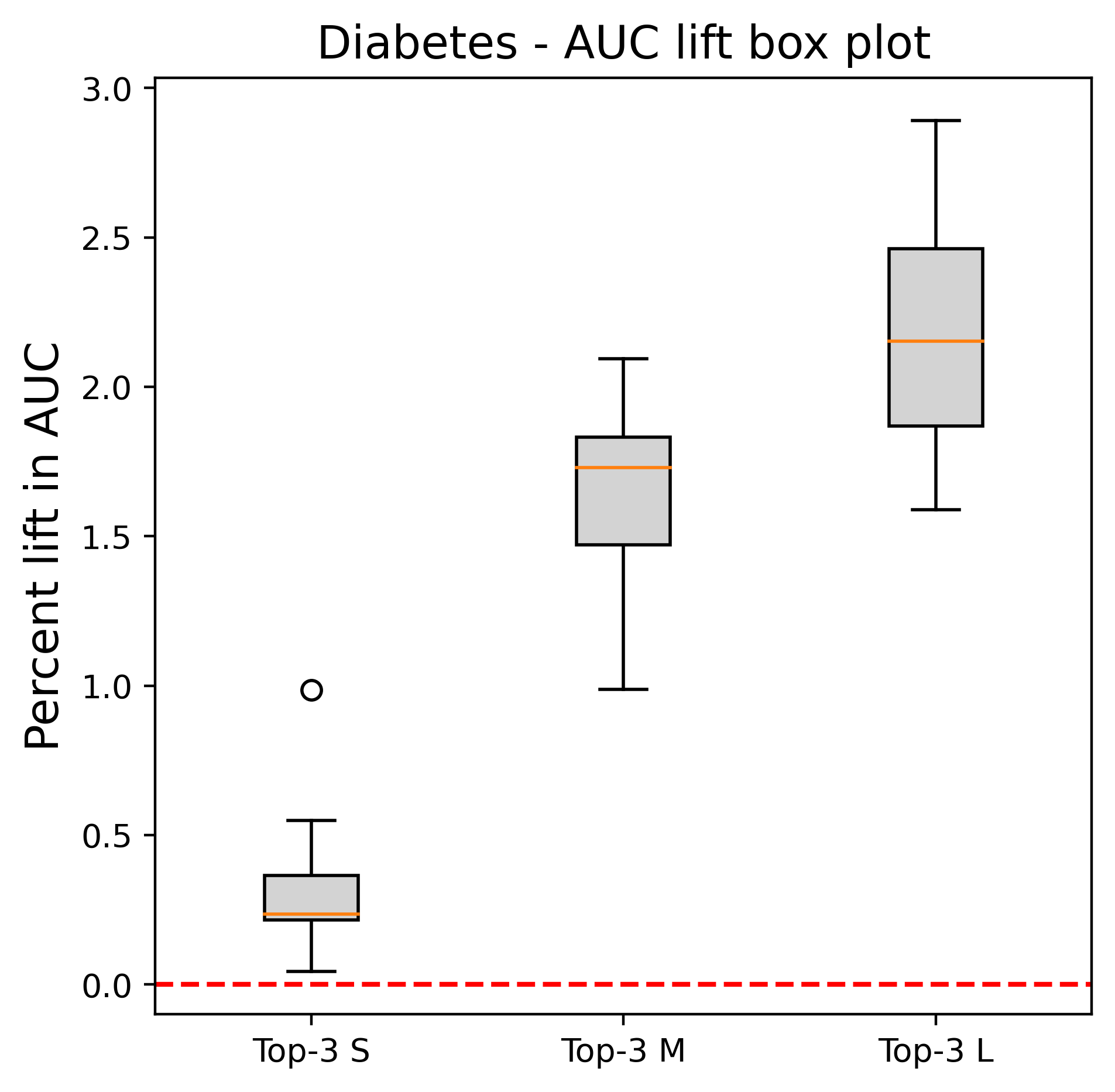}
\caption{Box plot for Diabetes - Test AUC lifts of our method for various runs of the top 3 models across model sizes.}
\label{fig:box_plot_diabetes}
\end{figure}

\begin{figure}[h!]
\centering
\includegraphics[width=0.45\textwidth]{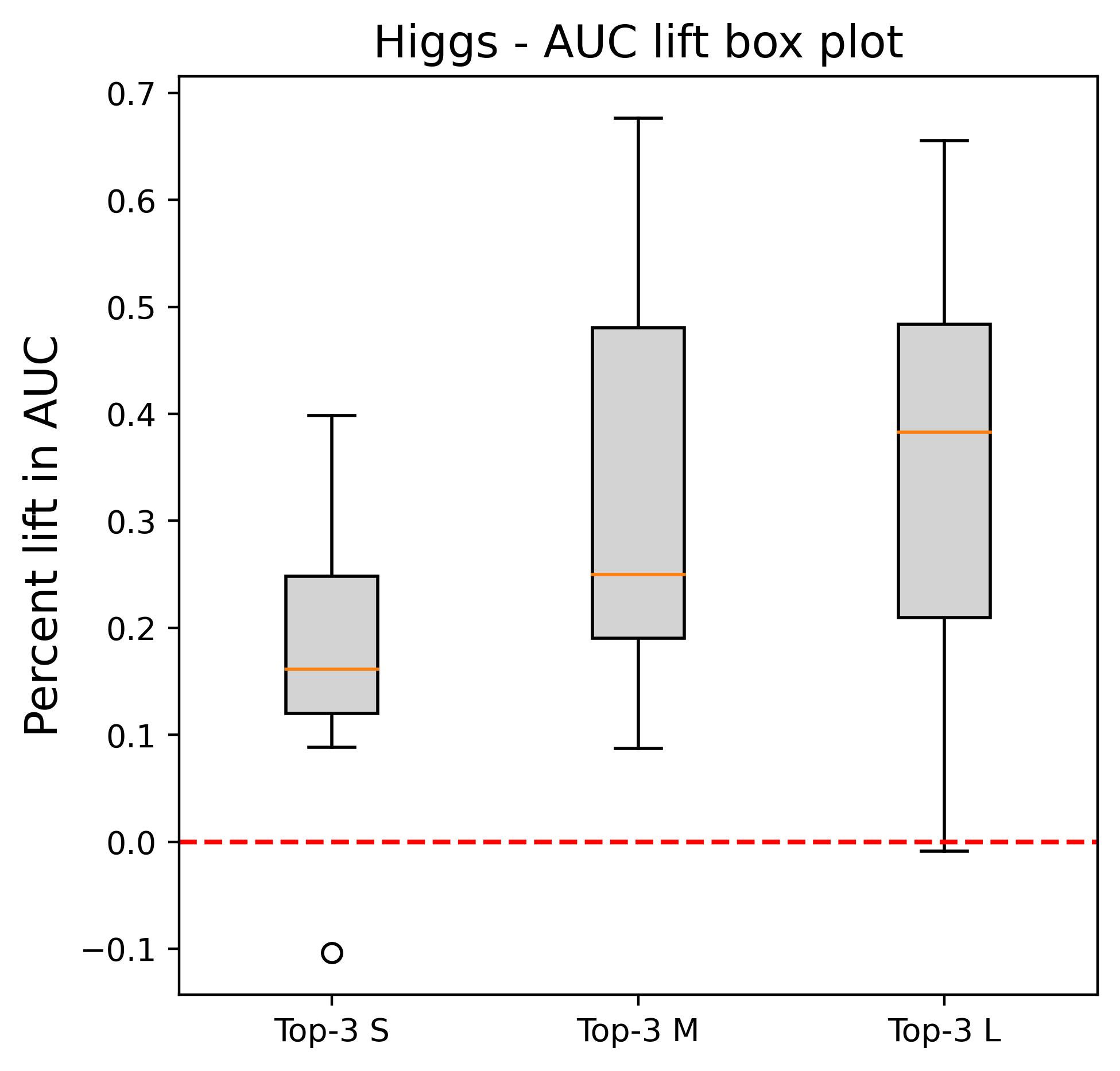}
\caption{Box plot for Higgs - Test AUC lifts of our method for various runs of the top 3 models across model sizes.}
\label{fig:box_plot_higgs}
\end{figure}

\end{document}